\documentclass[11pt]{article}
\usepackage{paper}

\title{Width-Robust Learnability in Mean-Field Bayesian Neural Networks}
\author{Dmitry Vaintrob\\Principles of Intelligence \and Kaarel H\"anni}
\date{June 2026}

\begin{document}
\maketitle

\begin{abstract}
Infinite-width limits are a standard way to reason about neural networks, but it is not automatic
that the limiting learner has the same complexity-theoretic inductive bias as large finite networks.
We study this question for Bayesian neural networks at the mean-field, or critical feature-learning,
scaling.  The central quantity is the \emph{reduced entropy}
\[
  s_\infty(y,\eps)=\limsup_N -\frac1N\log \pi_N^0(L\le \eps),
\]
the intensive prior cost of representing a target function $y$ to population mean-squared error
$\eps$.

Our main result is a width-robust learnability theorem. At fixed depth, a family of Boolean-cube
targets is learnable from polynomially many samples at infinite width if and only if it is learnable
at polynomial width, if and only if its reduced entropy is polynomially bounded. Equivalently, up to
polynomial slack in accuracy, the Bayesian mean-field learner generalizes exactly on the targets
that can be represented by polynomial-size networks.

The forward direction is proved by a form of subsampling: from the infinitely many hidden neurons
in the mean-field solution, one can select polynomially many representatives and still preserve the
learned function on every input simultaneously.  At the critical scaling this subsampling has both an
``active'' component, which keeps the data-dependent low-dimensional statistics, and a ``lazy''
component, which resamples the entropy-dominated directions from the prior.  Thus the infinite-width
mean-field limit gives a clean analytic description of learning without introducing spurious
width-dependent generalization power.
\end{abstract}

\tableofcontents
\bigskip

\section{Introduction}\label{sec:intro}

A wide neural network generalizes far better than its parameter count would suggest, because its
prior places more mass on simple, structured functions than on complex ones. Making this precise
means identifying the \emph{inductive bias} of the model: the complexity ordering on target
functions induced, through Bayesian conditioning, by the architecture together with the weight
prior. For a target $y$ the natural complexity parameter is the prior log-probability that the
network fits $y$ to a given accuracy --- the harder $y$ is to represent, the smaller this
probability and the larger the complexity. We study this complexity, and the learnability it
controls, in the overparameterized regime where the induced function-space bias is cleanest.

\begin{figure}[t]
\centering
\resizebox{0.78\textwidth}{!}{%
\begin{tikzpicture}[font=\sffamily,>={Stealth[length=3mm]}]
\def\bw{0.17}
\def\spec#1#2{%
  \foreach \h [count=\i from 0] in {1.0,0.82,0.68,0.56,0.46,0.38}{
    \pgfmathsetmacro\xx{#1-0.70+\i*0.245}
    \ifnum\i<#2\relax
      \fill[active] (\xx,0.62) rectangle (\xx+\bw,0.62+\h);
    \else
      \fill[lazy] (\xx,0.62) rectangle (\xx+\bw,0.62+\h);
    \fi
  }
}

\draw[ink,line width=0.9pt,->] (0.5,0.6) -- (10.8,0.6);
\node[ink,anchor=west,font=\sffamily] at (10.85,0.6) {$\gamma$};

\spec{2.0}{0}     
\spec{5.5}{2}     
\spec{9.0}{6}     

\node[ink,anchor=south,font=\sffamily\bfseries] at (5.5,1.78) {mean field};

\foreach \x in {2.0,5.5,9.0}{ \draw[ink,line width=0.8pt] (\x,0.6) -- (\x,0.46); }
\node[ink,anchor=north,font=\sffamily\small] at (2.0,0.42) {$\gamma<1$};
\node[ink,anchor=north,font=\sffamily\small] at (5.5,0.42) {$\gamma=1$};
\node[ink,anchor=north,font=\sffamily\small] at (9.0,0.42) {$\gamma>1$};

\node[anchor=north,font=\sffamily\footnotesize,text=active] at (5.0,-0.18) {active};
\node[anchor=north,font=\sffamily\footnotesize,text=ink] at (5.5,-0.18) {/};
\node[anchor=north,font=\sffamily\footnotesize,text=lazy] at (6.0,-0.18) {lazy};

\end{tikzpicture}
}
\caption{\textbf{The spectrum across scalings.} The same decaying feature-kernel spectrum, recoloured by regime. At $\gamma<1$ every direction is lazy (brown) and the learner is a fixed kernel (NTK/NNGP); at the mean-field scaling $\gamma=1$ a few top directions become active (blue) above a lazy bulk; at $\gamma>1$ the lazy bulk falls below the accuracy threshold and only active directions remain. This paper treats the critical scaling $\gamma=1$.}
\label{fig:spectra}
\end{figure}
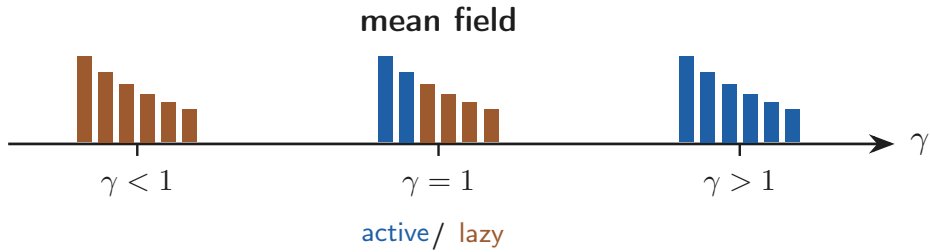

\paragraph{Which infinite width?}
Overparameterization is usually idealized by an infinite-width limit, but there are several such
limits and they disagree about complexity. In the \emph{lazy} (NTK/NNGP) scaling the network behaves
like a fixed kernel method \cite{jacot2018,lee2018}, with the inductive bias of a Gaussian process;
such methods provably require super-polynomially many samples on simple structured targets such as
parities \cite{danielymalach2020,barak2022}, even though those targets are computed by small
networks. The \emph{mean-field} or critical rich scaling ($\gamma=1$)
\cite{mmn2018,chizatbach2018,bordelonpehlevan2022,yanghu2021} keeps a different part of the finite
network alive. The output is still an average over many hidden neurons, but the distribution of
those neurons is allowed to change after seeing data. Thus the infinite-width object is not a frozen
kernel, but a law describing the typical learned hidden neuron. This limit is widely expected to
recover a more circuit-like complexity ordering. What has been missing is a generalization theory
for it --- and, just as importantly, an argument that the infinite-width limit says anything about
the finite widths one actually uses.

\paragraph{The reduced entropy.}
We make the complexity measure explicit. At the mean-field scaling a width-$N$ network has prior
probability $\pi^0_N(\loss\le\eps)$ of achieving population mean-squared error at most $\eps$ on $y$;
this probability decays like $e^{-N s}$, and the intensive rate
\[
  s_\infty(y,\eps)=\limsup_{N\to\infty}\Big(-\tfrac1N\log\pi^0_N(\loss\le\eps)\Big)
\]
is the \emph{reduced entropy} --- a width-independent complexity of $y$ relative to the prior. It
is the Bayesian/PAC-Bayes analogue of description length: the number of nats per neuron the prior
must spend to put mass near the target. It is finite for every $y$ (memorization gives
$s_\infty=O(2^d)$); the content is which targets have \emph{polynomial} reduced entropy.

\paragraph{Main result: learnability is width-robust.}
Our main result is that at $\gamma=1$ the reduced entropy governs learnability, and learnability is
the same at infinite width and at polynomial width. Informally (Theorem~\ref{thm:headline}): for a
family of targets $(y_d)_{d}$ and fixed depth, the following are equivalent --- learnable at
infinite width from polynomially many samples; learnable at polynomial width from polynomially many
samples; computed to $o(1)$ mean-squared error by a polynomial-width network; and
$s_\infty(y_d,\eps)$ polynomial in $d$. In a sentence: \emph{a target is learnable at infinite width
from polynomial data if and only if it is computed to $o(1)$ error by a polynomial-width network.}

This can be read as a complexity-theoretic sanity check on the mean-field limit. The infinite-width
equations replace a large finite network by a cleaner object: instead of tracking every neuron, one
tracks the probability law of a typical learned neuron, or equivalently a system of self-consistent
kernels and tilts. The theorem says that this cleaner object has not acquired an artificial
computational advantage merely from having infinitely many neurons available. Infinite width changes
the analytic description of learning, but not the class of targets learned with polynomial sample
complexity. The mechanisms visible in the mean-field equations are therefore not artifacts of an
infinite resource: any low-entropy mean-field solution has a polynomial-width realization with the
same predictions, up to polynomial slack.

\paragraph{Generalization is width-robust, and the witness is pointwise.}
The equivalence is, at heart, a robustness statement: the population-$\eps$ generalization behavior
of $\gamma=1$ mean-field Bayesian learning is insensitive to width. Exactly the same targets are
learnable at polynomial and at infinite width, up to the standard polynomial slack in the accuracy
thresholds, width, and sample size (Theorem~\ref{thm:nutsandbolts}).

What we prove is stronger than a coincidence of risk values. For any low-entropy infinite-width
Gibbs solution achieving $\eps$ accuracy, Theorem~\ref{thm:nutsandbolts} produces a single
polynomial-width network agreeing with it \emph{at every input simultaneously},
$\sup_x|f_{\hat\theta}(x)-f_\infty(x)|\le\delta$, and hence matching the target in mean square at
the same accuracy, up to the polynomial slack. So the function learned at infinite width is not
merely matched in mean square by a finite-width learner --- it is reproduced pointwise. The
width-robustness of generalization is witnessed by an explicit pointwise agreement between the
infinite- and polynomial-width learned functions, rather than inferred from two separately bounded
risks.

\paragraph{Relation to prior work.}
The result is closest in spirit to recent work showing that overparameterized neural networks can
have a function-space simplicity bias not visible from parameter counting alone. Buzaglo, Harel,
Nacson, Brutzkus, Srebro, and Soudry \cite{buzaglo2024} show that typical random interpolating
networks drawn from an apparently flat parameter prior generalize from narrow teachers, because
redundancy in the parameterization induces a non-uniform prior over functions. Daniely and Granot's
approximate-description-length framework gives another route to weakly width-dependent or
width-independent sample-complexity bounds for norm-controlled neural networks
\cite{danielygranot2019,danielygranot2024}. Our contribution is complementary: rather than bounding
a finite class directly, we identify the intensive posterior entropy of an infinite-width
mean-field Bayesian learner and prove that low-entropy infinite-width solutions compress to
polynomial width.

The mean-field side of the paper is connected to the growing literature on rich infinite-width
limits and adaptive kernels, including Tensor Programs/$\mu$P \cite{yanghu2021}, self-consistent GP
and DMFT descriptions of learned representations
\cite{navehringel2021,bordelonpehlevan2022,seroussi2023}, multi-scale adaptive-kernel theories
\cite{rubin2025kernels,lauditi2025}, and recent work of Ringel and collaborators on sample/width
scales and algorithmic capture in rich regimes \cite{rubin2025curse,davidovichringel2026}. The
lecture notes \cite{helias2026} give a broad Bayesian/statistical-physics account of the path from
Gaussian processes to adaptive learned kernels. We also share language with recent work of
Zavatone-Veth, Bordelon, and Pehlevan \cite{zavatoneveth2025summary}, where subsampling and summary
statistics are used to relate complicated learning dynamics to simpler effective descriptions. Our
theorem gives a complementary complexity statement for Bayesian fully connected networks on the
Boolean cube. Together with the companion analysis of the $\gamma>1$ regime and the standard
$\gamma<1$ reduction to kernel learning, this informally closes the width-scaling loop for this
model class: kernel behavior below the mean-field scale, critical mean-field behavior at
$\gamma=1$, and no additional polynomial-sample learnability gained merely by passing to infinite
width.

\paragraph{The engine: subsampling.}
The forward direction is a version of subsampling.  Starting from the infinite-width learned object,
we want to keep only polynomially many neurons and still get the same function.  If the neurons were
ordinary independent samples from a fixed distribution, this would be a routine Monte Carlo
statement.  The difficulty is that the distribution is itself learned from the training data, and in
principle the data could be spread thinly across very many weak directions.  The proof shows that,
for low reduced entropy solutions, this does not happen in a way that matters for the output.

At the critical scaling $\gamma=1$ this subsampling has two parts.  The first is an \emph{active}
subsampling step: after projecting to the few directions where the posterior has substantial
data-dependent content, we sample actual neurons and preserve their projected behavior. This is the
same kind of subsampling that appears in the over-rich regime $\gamma>1$ treated in the companion
paper. The second is a \emph{lazy} step: in the many remaining directions, we do not try to keep the
realized hidden coordinates. Instead we replace them by fresh prior noise and prove, uniformly over
all inputs, that the output barely changes. The comparison is a short Stein-type argument; see
Lemma~\ref{lem:swap}.

This active/lazy split is not just a technicality. At $\gamma=1$ the variational problem contains
three terms of the same order: the empirical loss, the weight prior, and the ordinary entropy of the
posterior over neurons. Active subsampling is the part needed to remember the low-dimensional
statistics selected by the data and the weight prior. Lazy subsampling is the part needed to keep
the entropy term honest: most hidden coordinates are not fixed by the data, but are resampled from
the correct prior law. In the over-rich regime $\gamma>1$, the entropy term becomes lower order and
the active subsampling picture is more direct. In the lazy/kernel regime $\gamma<1$, the weight-prior
cost is lower order in the corresponding variational problem, only lazy hidden-layer terms remain,
and the learner reduces to kernel learning.

This subsampling mechanism drives both directions of the equivalence. \emph{Compression} (the
forward direction) turns a low-entropy infinite-width network into a genuinely small one: keep
polynomially many representatives of the active behavior, resample the lazy behavior from the prior,
and obtain a network whose width is polynomial in the input dimension $d$ (with the degree set by
the depth) that agrees with the original network at every input at once
(Theorem~\ref{thm:A}).

The converse asks the reverse question --- must the reduced entropy be small whenever the target can
\emph{already} be computed by a small network? --- and shows that it must. Suppose some network of
modest width $\mathsf W$, a ``teacher'', already reaches the target accuracy. One can imitate it inside a much
wider network by \emph{cloning}: replace each of the teacher's $\mathsf W$ neurons by many independent neurons
whose weights sit near that neuron's, up to small noise. Since there are only $\mathsf W$ neurons to imitate,
this imitating distribution departs from the random-initialization prior in only a $\poly(\mathsf W)$-sized
way, so it already places appreciable probability on the accurate region --- and appreciable prior
probability is exactly what small reduced entropy means (Theorem~\ref{thm:cloning}). Compression and
cloning are inverse operations: a small network can be read off from low reduced entropy, and
conversely a small network forces the reduced entropy to be low.

The generalization bound and its converse then assemble these into the equivalence, with pointwise
compression supplying the fixed, finite-dimensional class of functions that the uniform-convergence
step relies on.

\begin{figure}[H]
\centering
\resizebox{0.72\textwidth}{!}{%
\begin{tikzpicture}[font=\sffamily,>={Stealth[length=2.8mm]}]

\def\sy{4.55}
\def\bw{0.34}
\foreach \i/\h [count=\n from 0] in {0/1.7,1/1.35,2/1.0}{
  \fill[active] ({2.30+\n*0.46},\sy) rectangle ({2.30+\n*0.46+\bw},\sy+\h);}
\foreach \i/\h [count=\n from 3] in {3/0.55,4/0.46,5/0.39,6/0.33,7/0.28,8/0.24,9/0.21}{
  \fill[lazy] ({2.30+\n*0.46},\sy) rectangle ({2.30+\n*0.46+\bw},\sy+\h);}
\draw[ink,line width=0.7pt] (2.20,\sy) -- (7.00,\sy);
\draw[ink,densely dashed,line width=0.7pt] (2.15,\sy+0.78) -- (7.05,\sy+0.78);
\node[ink,anchor=west,font=\sffamily\small] at (7.10,\sy+0.78) {$\lambda_0$};
\draw[active,line width=0.9pt,decorate,decoration={brace,amplitude=4pt,mirror}]
  (2.28,\sy-0.12) -- (3.62,\sy-0.12);
\node[active,anchor=north,font=\sffamily\footnotesize] at (2.95,\sy-0.22) {active $V$};
\draw[lazy,line width=0.9pt,decorate,decoration={brace,amplitude=4pt,mirror}]
  (3.86,\sy-0.12) -- (6.97,\sy-0.12);
\node[lazy,anchor=north,font=\sffamily\footnotesize] at (5.40,\sy-0.22) {lazy $C$};

\def\yc{1.9}
\def\xnet{1.0}
\foreach \yy in {0.55,1.05,1.55,2.05,2.55,3.05,3.55}{ \neuron{\xnet}{\yy}{1} }
\draw[ink,line width=0.8pt,decorate,decoration={brace,amplitude=4pt}]
  (\xnet-0.42,0.48) -- (\xnet-0.42,3.62);
\node[ink,anchor=east,font=\sffamily\footnotesize] at (\xnet-0.52,\yc) {$N\!\to\!\infty$};
\draw[boxred,dashed,line width=1pt,rounded corners=2pt]
  (\xnet-0.30,1.30) rectangle (\xnet+0.30,2.80);

\draw[active,line width=1.4pt,->] (\xnet+0.36,\yc) -- (4.15,\yc);
\node[active,anchor=south,font=\sffamily\scriptsize,inner sep=2pt] at (2.78,\yc+0.12) {project to active};
\def\xmid{4.6}
\foreach \yy in {1.55,2.05,2.55}{ \neuron{\xmid}{\yy}{0} }

\draw[lazy,line width=1.4pt,->] (\xmid+0.36,\yc) -- (7.25,\yc);
\node[lazy,anchor=south,font=\sffamily\scriptsize,inner sep=2pt] at (6.10,\yc+0.12) {fresh lazy chaos};
\def\xout{7.7}
\foreach \yy in {1.55,2.05,2.55}{ \neuron{\xout}{\yy}{1} }
\draw[ink,line width=0.8pt,decorate,decoration={brace,amplitude=4pt,mirror}]
  (\xout+0.34,1.38) -- (\xout+0.34,2.72);
\node[ink,anchor=west,font=\sffamily\footnotesize] at (\xout+0.44,\yc) {$N_0=\mathrm{poly}$};

\end{tikzpicture}
}
\caption{\textbf{Pointwise compression at $\gamma=1$.} \emph{Top:} the feature-kernel spectrum splits at the cutoff $\lambda_0$ into a low-dimensional active space $V$ ($\dim V\le1/\lambda_0$) and a lazy bulk $C$ ($\norm{C}_\op\le\lambda_0$). \emph{Bottom:} every neuron of the infinite-width learned network carries an active and a lazy part. The construction subsamples a few neurons (red), projects them onto the active directions through the fixed top eigenfunctions $\{\psi_k\}$ --- reconstructing the active block with no $O(1)$ bias --- and then resamples the lazy part as fresh prior chaos $N(0,C)$; the lazy-swap lemma (Lemma~\ref{lem:swap}) prices the resampling at $p\sqrt{\lambda_0}\,c_\star$ per consumed evaluation. The result is a width-$N_0=\poly$ network agreeing with $f_\infty$ at every input simultaneously.}
\label{fig:mechanism}
\end{figure}
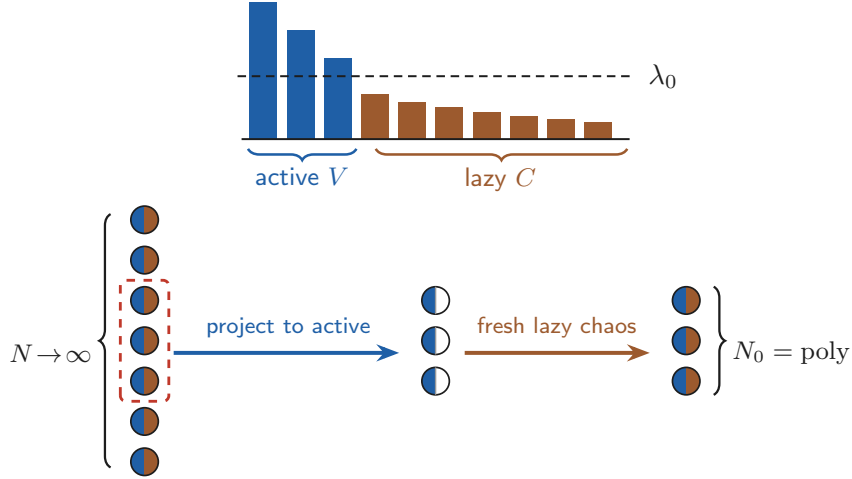

\paragraph{Assumptions and scope.}
The model assumptions are mild: a bounded Lipschitz activation with biases satisfying a quantitative
finite-set interpolation property, as holds for standard sigmoidal activations, and teachers
measured by width \emph{and} polynomial weight size. The mean-field limit itself --- the convergence
of the neuron distribution to a data-dependent law (propagation of chaos) --- is the standard object
of \cite{mmn2018,sirignano2020,nguyenpham2023,seroussi2023,rubin2024,lauditi2025}, which we recall
and cite rather than re-derive; the regularity properties our construction needs (moment and
sensitivity bounds on the conditioned neuron law) are \emph{derived} from the bounded activation and
a finite reduced entropy, not posited. We fix the depth $L$, and all polynomial bounds carry
exponents depending only on $L$. We do not claim that arbitrary finite training algorithms find the
Gibbs posterior, nor that the constants are practical; the result is a width-robust characterization
of the Bayesian mean-field inductive bias. This paper treats the infinite-width and polynomial-width
regimes; the general finite-width regime and the over-rich scalings $\gamma>1$ are the subject of a
companion paper.

\section{Setup}\label{sec:setup}

\subsection{Model and prior}

We fix an input dimension $d$ and the Boolean cube $\B^d=\{\pm1\}^d$ with the uniform measure
$\mu$; write $\calF=L^2(\B^d,\mu)$ and view targets and network functions as elements of
$\calF$. The target $y\colon\B^d\to[-1,1]$ is bounded, and $d$ is fixed within a complexity
picture (we never change it when varying the width).

We use a fully connected architecture of fixed depth $L$ ($L$ weight matrices, $L-1$ hidden
layers), with trainable biases,\footnote{Biases are included only for convenience: for depth
$L\ge2$ they are inessential whenever $\phi$ is neither even nor odd. The condition is necessary
--- a bias-free network with even (resp.\ odd) $\phi$ has $f(-x)=f(x)$ (resp.\ $f(-x)=-f(x)$) on
$\B^d$, and so cannot fit an antipodal pair $\{x,-x\}$ with unrelated targets --- and sufficient:
projecting onto a generic direction $a$ makes the values $a\cdot x_j$ distinct, and the dilates
$r\mapsto\phi(sr)$ ($s\in\R$) of a non-polynomial $\phi$ that is neither even nor odd are linearly
independent on $\{a\cdot x_j\}$, so suitable scales make $[\phi(s_i\,a\cdot x_j)]_{i,j}$ nonsingular
and Lemma~\ref{lem:memorize} holds with all biases set to zero. We keep them only to streamline the
exposition; the reader may safely ignore them.} hidden width $N$, and a single activation $\phi$. Hidden
pre/postactivations are $\zeta^\ell,z^\ell\in\R^N$ with $z^0=x$, $z^\ell=\phi(\zeta^\ell)$. We work
at the mean-field (Mei--Montanari--Nguyen) exponent $\gamma=1$: hidden weights carry the standard
$\mu$P normalization and the readout is an intensive $1/N$ average,
\begin{equation}\label{eq:forward}
  \zeta^{1}=W^1x+b^1,\qquad
  \zeta^{\ell+1}=\tfrac{1}{\sqrt N}\,W^{\ell+1}z^{\ell}+b^{\ell+1}\ \ (1\le\ell\le L-2),\qquad
  f_\theta(x)=\tfrac1N\sum_{i=1}^{N}\alpha_i\,\phi\big(\zeta^{L-1}_i(x)\big),
\end{equation}
with i.i.d.\ Gaussian prior weights $W^1_{ij}\sim\mathcal N(0,1/d)$,
$W^{\ell}_{ij}\sim\mathcal N(0,1)$, biases $b^\ell_i\sim\mathcal N(0,1)$, and readouts
$\alpha_i\sim\mathcal N(0,1)$; write $\pi^0_N$ for this prior on $\theta$. The exponent $\gamma=1$
is the scaling at which the readout is a genuine population average and the limit is
feature-learning rather than kernel ($\gamma<1$ collapses to the NNGP/kernel regime, $\gamma>1$ is a
companion paper).

Throughout we assume the activation $\phi$ is bounded ($|\phi|\le1$), Lipschitz, and
\emph{non-polynomial}. Boundedness and the Lipschitz bound are what the compression analysis uses;
non-polynomiality, together with the biases, is the classical condition under which the
architecture is a universal approximator \cite{leshno1993}, and we invoke it only through the
following interpolation bound. Teachers are measured by their width \emph{and} the size of their
weights: the complexity statements quantify over networks whose weights have polynomial size.

\begin{lemma}[Memorization]\label{lem:memorize}
Under the standing assumption, for any $\eps'>0$, any $n$ distinct points $x_1,\dots,x_n\in\B^d$ and
targets $t_1,\dots,t_n\in[-1,1]$, there is a one-hidden-layer network of width $O(n)$ whose weights
have size $\poly(n,d,1/\eps')$ and whose output is within $\eps'$ of $t_i$ at each $x_i$. Taking
$n=2^d$ gives an $\eps'$-accurate interpolant for every $y\colon\B^d\to[-1,1]$ at width $O(2^d)$
with weights of size $\poly(2^d,1/\eps')$.
\end{lemma}

\begin{proof}[Proof sketch]
Pick a direction $w$ with the $n$ values $w\cdot x_i$ distinct and separated by a gap
$\rho\ge1/\poly(n,d)$ (a generic, e.g.\ random, $w$ of norm $O(\sqrt d)$ works on the cube). Order
them and write $t(\cdot)$ as a sum of $n$ scaled increments. Each increment is realized to accuracy
$O(\eps'/n)$ by a single ridge feature $\phi(s(w\cdot x+b))$ whose transition scale resolves the gap,
with $s=\poly(1/\rho,n/\eps')$. The non-polynomial hypothesis with biases is used only to guarantee
this finite-set separation/approximation property; exact interpolation at $\eps'=0$ for arbitrary
non-polynomial Lipschitz activations is not needed anywhere below. The resulting weights are
$\poly(n,d,1/\eps')$.
\end{proof}

\subsection{Learning target and Bayesian MSE conditioning}

The learning target is the population mean-squared error
\[
  \loss(\theta)=\E_{x\sim\mu}\big(f_\theta(x)-y(x)\big)^2 .
\]
We study \emph{Bayesian MSE conditioning}: for an accuracy level $\eps>0$, the prior conditioned
on the MSE ball,
\begin{equation}\label{eq:posterior}
  \pi_\eps(y,N)\;=\;\pi^0_N\big(\,\cdot\mid \loss(\theta)\le\eps\,\big),
\end{equation}
and the (negative) \emph{reduced entropy} at $\gamma=1$ together with its infinite-width limit,
\begin{equation}\label{eq:redent}
  s_N(y,\eps)=-\tfrac1N\log\pi^0_N\big(\loss(\theta)\le\eps\big),
  \qquad
  s_\infty(y,\eps)=\limsup_{N\to\infty}s_N(y,\eps).
\end{equation}
Equivalently $N\,s_N(y,\eps)=\KL\!\big(\pi_\eps(y,N)\,\Vert\,\pi^0_N\big)$, since restricting and
renormalizing to an event of mass $p$ costs relative entropy $-\log p$.

The sharp conditioning \eqref{eq:posterior} and its smooth counterpart play distinct roles, and it
is worth separating them at the outset. The hard ball \eqref{eq:posterior} serves only to
\emph{define} the complexity measure: its reduced entropy $s_\infty$ is the quantity the equivalence
is stated in. The object the mean-field formalism actually acts on, and the object that is trained
and shown to generalize, is the smooth \emph{Gibbs} (tempered) posterior
$\pi_\beta\propto e^{-\beta\loss}\,d\pi^0_N$, tuned so that its expected loss equals $\eps$, with
intensive inverse temperature $\bar\beta_\star=\beta/N$ --- at finite data its empirical form
$\pi_\beta\propto e^{-\beta\widehat{\loss}_D}\,d\pi^0_N$, the learner of \S\ref{sec:gen}. The smooth tilt is
what carries a score for the mean-field equations (Definition~\ref{def:tilt}), where the hard
indicator carries none; the hard ball enters the learner only through the budget $s_\infty$, never as
a constraint imposed on it. The single fact we need relating the two is that the required inverse
temperature is polynomially bounded (Lemma~\ref{lem:multiplier}).

\section{Main results}\label{sec:main}

We state the results before developing the mean-field formalism they are proved in
(\S\ref{sec:limit} onward). The first is the headline equivalence; the second is the engine.

\begin{theorem}[Width-robust learnability; informal]\label{thm:headline}
Fix depth $L$ and a family $(y_d)$, $y_d\colon\{\pm1\}^d\to[-1,1]$. Under $\gamma=1$ Bayesian MSE
learning with a bounded, Lipschitz, non-polynomial activation (with biases), the following are
equivalent, where each ``polynomial'' means a fixed polynomial in $d$ with exponent depending only
on $L$:
\begin{enumerate}[label=\textup{(\roman*)},nosep,topsep=2pt]
\item \textup{(infinite width)} there is polynomial data $D(d)$ such that the infinite-width
  \emph{Gibbs posterior} trained on $D$ random points has population MSE $\loss=o(1)$ in expectation;
\item \textup{(polynomial width)} there are polynomial $N_0(d),D(d)$ such that the width-$N_0$
  Gibbs posterior trained on $D$ random points has population MSE $\loss=o(1)$;
\item there is a polynomial-width network, with weights of polynomial size, achieving $\loss=o(1)$ on $y_d$.
\end{enumerate}
Equivalently, the reduced entropy $s_\infty(y_d,\eps)$ is polynomially bounded in $d$. Informally:
\emph{a target is learnable at infinite width from polynomial data iff it is computed to $o(1)$
mean-squared error by a polynomial-width network.}
\end{theorem}

\noindent The label \emph{informal} refers only to the phrasing --- the asymptotic family $(y_d)$
and the ``$\loss=o(1)$'' shorthand. The statement is a genuine theorem: it is the assembly
(\S\ref{sec:assembly}) of the quantitative compression and conservation bounds
(Theorem~\ref{thm:nutsandbolts}) with the generalization and converse results
(Theorems~\ref{thm:gen} and~\ref{thm:conv}), all proved below. The ``posterior'' in (i)--(ii) is the
empirical Gibbs posterior $\pi\propto\exp(-N\bar\beta\,\widehat{\loss}_D)\,\pi^0_N$ at a polynomial
intensive inverse temperature $\bar\beta$ chosen by the free-energy bounds
(Lemmas~\ref{lem:multiplier} and~\ref{lem:multiplier-app}); the reduced-entropy budget enters only
through the complexity side $s_\infty$, never as a hard constraint on the learner.

\paragraph{From the equivalence to two primitives.}
Both directions of Theorem~\ref{thm:headline} factor through a single pair of quantitative statements
about a \emph{fixed} target at a \emph{fixed} accuracy. In one direction, a small accurate network
must make the accurate region easy to reach for the prior --- it must force the reduced entropy to be
small; we call this \emph{conservation} (the small network is ``cloned'' into the wide one without
spending much entropy). In the other, small reduced entropy must be realizable by a genuinely small
network that still generalizes; we call this \emph{compression} (the wide accurate network is
subsampled down to polynomial width). To run these at finite accuracy we allow a little slack:
compression starts from a network accurate to an inner level $\eps_-$ and returns one accurate to an
outer level $\eps_+>\eps_-$, losing a controlled amount along the way, while conservation runs in the
reverse direction. The gap $\Delta=\min(\eps-\eps_-,\eps_+-\eps)$ is the room to maneuver, and it is
what the polynomial bounds depend on. These are the two primitives; \S\ref{sec:assembly} chains them
with the generalization bound (Theorem~\ref{thm:gen}) and its converse (Theorem~\ref{thm:conv}) to
recover Theorem~\ref{thm:headline}.

\begin{theorem}[Compression and conservation]\label{thm:nutsandbolts}
Fix $L$ and accuracy levels $\eps_-<\eps<\eps_+$ with gap $\Delta=\min(\eps-\eps_-,\eps_+-\eps)$. For any $y\colon\B^d\to[-1,1]$:
\begin{enumerate}[label=\textup{(\alph*)},nosep,topsep=2pt]
\item \textup{(compression)} if $s_\infty(y,\eps_-)\le s$, there is a network $f_{\hat\theta}$ of
  width $N_0\le\poly(s,d,1/\Delta)$ (exponent $O(L)$) and weights of size $\poly(s,d,1/\Delta)$
  (exponent $O(L)$) with population MSE $\loss(\hat\theta)\le\eps_+$; moreover $f_{\hat\theta}$ agrees
  \emph{pointwise} with an infinite-width MSE-$\eps$-accurate network,
  $\sup_{x\in\B^d}|f_{\hat\theta}(x)-f_\infty(x)|\le\delta$;
\item \textup{(conservation)} if some width-$\mathsf W$ network with weights of polynomial size has
  $\loss\le\eps_-$, then $s_\infty(y,\eps)\le\poly(\mathsf W,d)$.
\end{enumerate}
In particular $s_\infty(y,\eps)\le\poly(2^d)$ for every $y$ (Lemma~\ref{lem:memorize}).
\end{theorem}

\section{The mean-field limit and its consequences}\label{sec:meanfield}

\subsection{The infinite-width limit}\label{sec:limit}

As $N\to\infty$ the network admits the standard mean-field description, which we take as a cited
input and recall in the two features we use. Existence and structure in the feature-learning
regime are developed in \cite{mmn2018,sirignano2020,chizatbach2018,rotskoff2022,nguyenpham2023,yanghu2021,bordelonpehlevan2022}
and, in the Bayesian/posterior form relevant to \eqref{eq:posterior}, in
\cite{seroussi2023,rubin2024,lauditi2025}.

\begin{definition}[Single-site mean-field description]\label{def:chaos}
In the limit the posterior is described by a tower of \emph{single-site laws} $P_\ell$, one per
layer, on preactivation fields $\zeta\in\calF$ --- the function $x\mapsto\zeta(x)$ on $\B^d$,
equivalently a vector in $\R^{2^d}$, carrying the scalar readout $\alpha$ at the last hidden layer.
Any fixed finite set of neurons is an independent sample from $P_\ell$ (\emph{propagation of
chaos}), $\tfrac1m\sum_{i=1}^m\Psi(\zeta_i)\to\E_{P_\ell}\Psi$ for bounded continuous $\Psi$, so a
width-$N$ network is $N$ i.i.d.\ draws and its output is the single-site average
$f(x)=\E_{P_{L-1}}[\alpha\,\phi(\zeta(x))]$. The layers are linked by their \emph{feature kernels}
\begin{equation}\label{eq:kernel}
  K_\ell(x,x')=\E_{P_\ell}\!\big[\phi(\zeta(x))\,\phi(\zeta(x'))\big],
  \qquad K_\ell(x,x)\le1,\quad \tr K_\ell\le1,
\end{equation}
in the sense that, conditionally on the previous layer, a layer-$(\ell{+}1)$ preactivation field is
the Gaussian with covariance $K_\ell$.
\end{definition}

\begin{definition}[Boltzmann tilt]\label{def:tilt}
Conditioning on the data tilts each single-site law by a self-consistent effective potential: the
posterior single-site law is a reweighting of the prior single-site law,
$dP_\ell\propto e^{-g_\ell}\,dP^0_\ell$, with $g_\ell$ fixed by the mean-field equations. At the
readout the tilt is generated by the loss, $g=\beta\,\ell_{\mathrm{eff}}$ with $\beta$ the
inverse temperature/multiplier conjugate to the accuracy constraint; at hidden layers $g_\ell$ is
the backpropagated molecular field. We call $\norm{\nabla g_\ell}_{L^2(\mu)}$ the
\emph{susceptibility} of layer $\ell$.
\end{definition}

In general the infinite-width limit is a \emph{mixture} of such towers, and the mean-field equations fixing the tilt are stated at the level of a single pure component. Appendix~\ref{app:mf} makes this precise: each component is a mean-field \emph{vacuum} $\mathfrak v$, a self-consistent solution of the mean-field equations that supplies the kernels $K_\ell$ and the tilt $g_\ell$ above. Every estimate below holds conditionally on a vacuum, with constants depending only on the KL/free-energy budget, and is then integrated over the mixture; we therefore suppress the vacuum and work with a single tower $P_\ell$.

\paragraph{The Cameron--Martin norm.}
For a layer kernel $K_\ell$ write $\cm{w}^2=\ip{w}{K_\ell^{-1}w}$ for the Cameron--Martin norm of
the prior field $\mathcal N(0,K_\ell)$ --- the reproducing-kernel Hilbert space norm of $K_\ell$,
which weights each direction inversely by its prior variance, so a field has small $\cm{\cdot}$
exactly when it is concentrated on the high-variance (top-spectrum) directions of the kernel. Write
$e_x\in\calF$ for the evaluation functional, $\ip{w}{e_x}=w(x)$, so that
$K_\ell(x,x)=\ip{e_x}{K_\ell e_x}$.

\subsection{First consequences of finite reduced entropy}\label{sec:consequences}

Three quantities the construction needs are bounded here from the standing assumption and a finite
reduced entropy: the inverse temperature conjugate to the accuracy constraint, the readout second
moment, and the layer susceptibilities (the lazy-coordinate scores). The first makes precise the
``polynomial flexibility on the bounds'' and is the only place the gap $\Delta$ enters.

\begin{lemma}[Multiplier bound]\label{lem:multiplier}
Fix $\eps_-<\eps$ and suppose $s_\infty(y,\eps_-)\le K$. Let $F(\beta)=-\log\E_{\pi^0_N}e^{-\beta\loss}$
and let $\beta_\star$ be the inverse temperature at which the Gibbs posterior
$\pi_\beta\propto e^{-\beta\loss}\pi^0_N$ has expected loss $\E_{\pi_{\beta_\star}}[\loss]=\eps$. Then the
intensive multiplier --- the $\beta$ of the readout tilt in Definition~\ref{def:tilt}, with the
extensive temperature $N$ divided out --- obeys
\[
  \bar\beta_\star:=\beta_\star/N\;\le\;\frac{s_N(y,\eps_-)}{\eps-\eps_-}
  \;\xrightarrow{\,N\to\infty\,}\;\frac{K}{\eps-\eps_-}\;\le\;\frac{K}{\Delta}.
\]
\end{lemma}

\begin{proof}
For every $\beta\ge0$,
$\E_{\pi^0_N}e^{-\beta\loss}\ge\int_{\{\loss\le\eps_-\}}e^{-\beta\loss}\,d\pi^0_N\ge e^{-\beta\eps_-}\pi^0_N(\loss\le\eps_-)$,
so taking $-\log$,
\begin{equation}\label{eq:chernoff}
  F(\beta)\;\le\;\beta\eps_-+\big(-\log\pi^0_N(\loss\le\eps_-)\big)\;=\;\beta\eps_-+N\,s_N(y,\eps_-).
\end{equation}
On the other hand, with $Z_\beta=\E_{\pi^0_N}e^{-\beta\loss}=e^{-F(\beta)}$,
\[
  0\le\KL(\pi_{\beta_\star}\Vert\pi^0_N)=\E_{\pi_{\beta_\star}}[-\beta_\star\loss-\log Z_{\beta_\star}]
  =F(\beta_\star)-\beta_\star\,\E_{\pi_{\beta_\star}}[\loss]=F(\beta_\star)-\beta_\star\eps,
\]
i.e.\ $F(\beta_\star)\ge\beta_\star\eps$. Combining with \eqref{eq:chernoff} at $\beta=\beta_\star$,
$\beta_\star\eps\le\beta_\star\eps_-+N s_N(y,\eps_-)$, hence
$\beta_\star(\eps-\eps_-)\le N s_N(y,\eps_-)$ and $\bar\beta_\star\le s_N(y,\eps_-)/(\eps-\eps_-)$;
let $N\to\infty$. No tightness of \eqref{eq:chernoff}, and no comparison of active and Gibbs
entropies, is used.
\end{proof}

\noindent We take $K=s_\infty(y,\eps_-)$, the reduced entropy at the \emph{inner} radius, as the
budget for all bounds below; we do not assume any comparability of $s_\infty$ between the accuracy levels. Thus
$\bar\beta_\star\le K/\Delta$. We now use this for the susceptibility.

\begin{proposition}[Regularity of a finite-KL Gibbs posterior]\label{prop:reg}
Let $\pi_{N,\bar\beta}$ be either the population Gibbs posterior or an empirical Gibbs posterior at
intensive inverse temperature $\bar\beta$, and suppose
\[
  \tfrac1N\KL(\pi_{N,\bar\beta}\Vert\pi^0_N)\le R,
  \qquad \bar\beta\le R .
\]
Let $P_\ell$ be any single-site law of the mean-field description (Appendix~\ref{app:mf}).  In the infinite-width limit, after removing a
posterior subpopulation whose contribution to the downstream suffix output is $o(1)$ in sup norm,
the following quantities are bounded by $\poly(R,d)^{O(L)}$:
\begin{enumerate}[label=\textup{(\alph*)},nosep,topsep=3pt]
\item \textup{(readout moment and cap)} the last-layer readout coordinate satisfies
  $\E_{P_{L-1}}[\alpha^2]\le O(1+R)$, and capping $|\alpha|$ at a polynomial threshold changes the
  output by $o(1)$;
\item \textup{(operator norm)} the kept incoming weight matrices obey
  $\tfrac1{\sqrt N}\norm{W^{\ell+1}_{\mathrm{kept}}}_\op\le\poly(R,d)^{O(1)}$ with high probability;
\item \textup{(active coefficients)} for every polynomial-dimensional active space $U_\ell$, the kept
  fields have polynomial Cameron--Martin coefficients
  $\norm{P_{U_\ell}K_\ell^{-1/2}\zeta}\le\poly(R,d)^{O(1)}$;
\item \textup{(susceptibility)} the backpropagated lazy-score through any already constructed
  polynomial regular suffix is bounded by
  \[
    \sup_h\E_{\rho_h}\norm{\nabla_{\rm lazy}g_\ell}_{L^2(\mu)}
    \le c_\star=\poly(R,d)^{O(L)} .
  \]
  For empirical loss, the corresponding lazy-swap input is bounded by the empirical lazy-Gram
  estimate of Lemma~\ref{lem:softgram} once the empirical Gram top space is included in $U_\ell$.
\end{enumerate}
\end{proposition}

\noindent This is the body-level form of Proposition~\ref{prop:package} and
Proposition~\ref{prop:trunc}.  The point is that the regularity assumptions are derived from a
Gibbs KL-rate bound, not from support on a hard ball: Lemma~\ref{lem:ss-kl} transfers the KL rate to
single-site marginals, Gaussian entropy inequalities give moment and operator-norm truncations, and
the suffix adjoint invariant bounds the actual contribution of the truncated population to the
output layer by layer.  Reduced entropy enters only by Lemma~\ref{lem:multiplier} or the empirical
free-energy bound of Lemma~\ref{lem:multiplier-app}, which provide the required polynomial value of
$R$.

\subsection{Spectral split}\label{sec:split}

Fix a layer kernel $K=K_\ell$ and diagonalize it, $K=\sum_k\lambda_k\,\psi_k\otimes\psi_k$ with
$\lambda_1\ge\lambda_2\ge\cdots\ge0$ and $\{\psi_k\}$ orthonormal in $\calF$. For a cutoff
$\lambda_0\in(0,1)$ split the spectrum into an \emph{active} (top) part and a \emph{lazy} (bulk)
part,
\[
  V=\operatorname{span}\{\psi_k:\lambda_k\ge\lambda_0\},\qquad
  C=K\big|_{V^\perp}=\!\!\sum_{\lambda_k<\lambda_0}\!\!\lambda_k\,\psi_k\otimes\psi_k .
\]
Because $\tr K\le1$ \eqref{eq:kernel}, the active part is low-dimensional,
$\dim V\le\lambda_0^{-1}$ (each retained eigenvalue is $\ge\lambda_0$ and they sum to $\le1$),
and $\norm{C}_\op\le\lambda_0$. The single estimate the lazy-swap rests on is that the lazy block
is small \emph{uniformly in the input}: with $e_x$ the evaluation functional,
\begin{equation}\label{eq:Cdiag}
  C^2(x,x)=\!\!\sum_{\lambda_k<\lambda_0}\!\!\lambda_k^2\,\psi_k(x)^2
  \;\le\;\lambda_0\!\!\sum_{\lambda_k<\lambda_0}\!\!\lambda_k\,\psi_k(x)^2
  \;=\;\lambda_0\,C(x,x)\;\le\;\lambda_0,
\end{equation}
using $C(x,x)\le K(x,x)\le1$. Equivalently
\begin{equation}\label{eq:Cunif}
  \norm{C^{1/2}e_x}_{L^2(\mu)}=\sqrt{C(x,x)}\le1
  \quad\text{and}\quad
  \norm{C^{1/2}}_\op=\sqrt{\lambda_{\max}(C)}\le\sqrt{\lambda_0},
  \qquad\text{both uniform in }x .
\end{equation}
The active part will be reproduced explicitly by the subsampled neurons; the lazy part is what the
next lemma swaps for fresh prior noise. The names echo the training regimes: a \emph{lazy} direction is one left at the prior, behaving as in the lazy (kernel) regime, while an \emph{active} direction carries data-dependent feature learning. This is the same dichotomy at the level of single directions --- at $\gamma<1$ there are no active directions, so every hidden layer is a fresh prior draw and the learner is exactly kernel (NNGP) regression; the content of $\gamma=1$ is that a low-dimensional set of directions turns active while the rest stay lazy, which is what the compression of \S\ref{sec:compression} keeps while resampling the rest from the prior.

\section{Pointwise compression}\label{sec:compression}

\subsection{The lazy-swap lemma}\label{sec:swap}

The lemma compares the tilted lazy law to the prior lazy law against any bounded Lipschitz
functional of finitely many input evaluations, with an error governed by the lazy susceptibility
(Proposition~\ref{prop:reg}) times the lazy scale $\sqrt{\lambda_0}$
\eqref{eq:Cunif}. Because the tilt couples the active and lazy coordinates, the statement is
\emph{conditional on the active coordinates} $h$ and uniform in them. It is the mechanism behind ``the
loss couples only to the top directions'': along the lazy block the tilt is nearly invisible. We
prove it by Stein's method, which needs no sub-Gaussianity, moment matching, or convexity, and which
delivers the multi-point version the construction consumes.

\begin{lemma}[Lazy swap, multi-point]\label{lem:swap}
Fix the active coordinates $h\in V$ and let $\rho_h(dv)\propto e^{-g(h+v)}\rho_0(dv)$ on $V^\perp$,
$\rho_0=\mathcal N(0,C)$, with lazy susceptibility
$\E_{\rho_h}\norm{\nabla_v g(h+\cdot)}_{L^2(\mu)}\le c$. Let $F(v)=G\big(v(x_1),\dots,v(x_p)\big)$
depend on the field through $p$ evaluations at \emph{arbitrary} inputs, with $G$ bounded and
$1$-Lipschitz in each argument. Then
\[
  \big|\E_{\rho_h}F-\E_{\rho_0}F\big|\;\le\;p\,\sqrt{\lambda_0}\,c ,
\]
uniformly in $h$ and in the evaluation points. In particular the one- and two-point cases give
$\big|\E_{\rho_h}\phi(a+v(x))-\E_{\rho_0}\phi(a+v(x))\big|\le\sqrt{\lambda_0}\,c$ and
$\big|\E_{\rho_h}[\phi_x\phi_{x'}]-\E_{\rho_0}[\phi_x\phi_{x'}]\big|\le2\sqrt{\lambda_0}\,c$. Applied
with $c=c_\star$ from Proposition~\ref{prop:reg}, the error is
$\poly(K,d,1/\Delta)^{O(L)}\sqrt{\lambda_0}$.
\end{lemma}

\begin{proof}
\emph{Whitening.} Let $\tilde v=C^{-1/2}v$, standard Gaussian under $\rho_0$, and
$\tilde g(\tilde v)=g(h+C^{1/2}\tilde v)$, so $\nabla_{\tilde v}\tilde g=C^{1/2}\nabla_v g(h+v)$ and,
by \eqref{eq:Cunif}, $\norm{\nabla_{\tilde v}\tilde g}\le\norm{C^{1/2}}_\op\norm{\nabla_v g}\le
\sqrt{\lambda_0}\,\norm{\nabla_v g}$ pointwise in $\tilde v$. Each evaluation is a linear form
$v(x_i)=\ip{\tilde v}{b_i}$ with $b_i=C^{1/2}e_{x_i}$, $\norm{b_i}=\sqrt{C(x_i,x_i)}\le1$
\eqref{eq:Cunif}, so $\tilde F(\tilde v)=G(\ip{\tilde v}{b_1},\dots,\ip{\tilde v}{b_p})$ has
$\norm{\nabla\tilde F}_\infty\le\sum_i\norm{b_i}\le p$.

\emph{Stein solution.} Let $\mathcal Lf=\Delta f-\ip{\tilde v}{\nabla f}$ be the Ornstein--Uhlenbeck
generator (invariant law $\rho_0$ in the $\tilde v$ coordinate), and $f$ solve
$\mathcal Lf=\tilde F-\E_{\rho_0}\tilde F$, i.e.\ $f=\int_0^\infty(\E_{\rho_0}\tilde F-P_t\tilde F)\,dt$.
Mehler's formula $\nabla P_t=e^{-t}P_t\nabla$ gives $\norm{\nabla f}_\infty\le\norm{\nabla\tilde F}_\infty\le p$.

\emph{One integration by parts.} Under $\rho_h\propto e^{-\tilde g}\rho_0$, Gaussian integration by
parts gives $\E_{\rho_h}[\mathcal Lf]=\E_{\rho_h}\ip{\nabla f}{\nabla\tilde g}$, and since
$\E_{\rho_h}[\mathcal Lf]=\E_{\rho_h}\tilde F-\E_{\rho_0}\tilde F$,
\[
  \big|\E_{\rho_h}F-\E_{\rho_0}F\big|=\big|\E_{\rho_h}\ip{\nabla f}{\nabla\tilde g}\big|
  \le\norm{\nabla f}_\infty\;\E_{\rho_h}\norm{\nabla\tilde g}
  \le p\,\sqrt{\lambda_0}\;\E_{\rho_h}\norm{\nabla_v g}\le p\,\sqrt{\lambda_0}\,c. \qedhere
\]
The bound uses the score only through its \emph{$\rho_h$-expectation} $\E_{\rho_h}\norm{\nabla_v g}\le c$,
not a deterministic sup.
\end{proof}

\begin{remark}[Where the uniformity comes from]\label{rem:unif}
The factor $\sqrt{\lambda_0}$ is $\norm{C^{1/2}}_\op$ and the per-point $1$ is
$\norm{C^{1/2}e_x}=\sqrt{C(x,x)}\le1$; both are uniform in $x$ by \eqref{eq:Cunif}, which is the
entire reason the conclusion is a sup-norm statement rather than a $\mu$-averaged one. This
uniformity is what upgrades MSE accuracy to the pointwise agreement of
Theorem~\ref{thm:nutsandbolts}(a).
\end{remark}

\begin{remark}[What Stein removed]\label{rem:stein}
The comparison reduces to the exact identity $\E_\rho F-\E_{\rho_0}F=\E_\rho\ip{\nabla f}{\nabla\tilde g}$;
the Mehler bound caps $\norm{\nabla f}_\infty$ by the test functional's Lipschitz constant, and
$\sqrt{\lambda_0}=\norm{C^{1/2}}_\op$ supplies the smallness. No sub-Gaussianity, no moment
matching, and no convexity of the tilt are used, and the same line gives the two-point (kernel)
and general $p$-point versions at the cost of the factor $p$ --- which is exactly what the
construction consumes when it reproduces both features and the diagonal of the lazy kernel.
\end{remark}

\subsection{The construction and the compression theorem}\label{sec:constr}

Fix the cutoff $\lambda_0=(\delta/\poly)^2$ and subsample $N_0=\poly(K,d,1/\delta)^{O(L)}$ neurons
from each single-site law $P_\ell$, together with fresh independent noise; assemble a width-$N_0$
network $\hat\theta$ layer by layer. Figure~\ref{fig:mechanism} summarizes the construction. The first layer is copied verbatim. At a hidden layer each
incoming weight is built in two pieces, matched to the $\mu$P forward pass
$\zeta^{\ell+1}=\frac1{\sqrt{N_0}}W\hat h$, where $\hat h=\phi(\zeta^\ell)$ is the postactivation
feature of the previous layer.

\emph{Active block, via fixed eigenfunctions.} Read the kept features against the \emph{deterministic} top
eigenfunctions $\{\psi_k\}_{k\in V}$ of $K_\ell$ --- not a sample-dependent pseudo-inverse. Since
$\E_{P_\ell}[\ip{\hat h}{\psi_k}\hat h]=K_\ell\psi_k=\lambda_k\psi_k$ lands back in $V$, the empirical
average $\frac1{N_0}\sum_i\ip{\hat h_i}{\psi_k}\hat h_i$ reproduces the active block $P_V\zeta_j$ up to a
\emph{mean-zero} $O(\poly/\sqrt{N_0})$ fluctuation, with no $O(1)$ lazy bias --- the defining advantage
of reading against fixed directions. The weight size is $\norm{\widehat W^{A}_j}_2^2\to
\cm{P_V\zeta_j}^2\le\dim V+O(K)\le\poly$ (the budget $K$ caps the Cameron--Martin norm on $V$).

\emph{Lazy directions, from the prior.} On the kept neurons draw a fresh prior Gaussian weight and project it off the active read directions, $\sum_i\eta_{ji}P_V\hat h^\ell_i=0$ --- equivalently, keep only its component in the kernel of the $\dim V$-row active map, still a genuine Gaussian weight. The resulting field has \emph{zero} active component and lazy covariance $\to C_\ell$: the lazy block of each neuron is drawn straight from the prior law $N(0,C_\ell)$, while the active block is supplied by the reconstruction above. Only $r=\dim V$ of the $N$ directions are projected out --- a vanishing fraction --- so the weight is the prior up to that projection; the projection is exactly what stops the draw from leaking an $O(1)$, unswappable component into the active block, where (unlike the lazy block, with $\norm{C}_\op\le\lambda_0$) the swap has no smallness to spend.

The lazy-swap lemma then prices the one move that changes the law --- replacing the realized lazy
block by this fresh chaos --- at $\le p\,\sqrt{\lambda_0}\,c_\star$ per consumed $p$-point evaluation,
with $c_\star$ the susceptibility of Proposition~\ref{prop:reg}. Writing $e_\ell(x)$ for the
layer-$\ell$ per-input active-reconstruction error, one step contributes the active fluctuation
($\poly/\sqrt{N_0}$, mean-zero), the lazy swap ($\poly\sqrt{\lambda_0}\,c_\star$), and the inherited
error read through bounded coefficients and the $1$-Lipschitz $\phi$, giving the recursion
\begin{equation}\label{eq:recursion}
  e_{\ell+1}(x)\le\poly\Big(\sqrt{\lambda_0}+\tfrac1{\sqrt{N_0}}+e_\ell(x)\Big),
  \qquad e_1(x)\le\poly/\sqrt{N_0},
\end{equation}
uniformly in $x$, with all Hoeffding events unioned over the $2^d$ inputs, the $\dim V\le\poly$
directions, the $O(L)$ layers and the $N_0$ neurons (an $o(1)$ total once $N_0\ge\poly\cdot d$).

\begin{theorem}[Pointwise compression; proof of Theorem~\ref{thm:nutsandbolts}(a)]\label{thm:A}
Assume $s_\infty(y,\eps_-)\le s$ and the standing assumption on $\phi$, with fixed depth $L$. For
every $\delta\in(0,1)$ the construction above returns a width-$N_0$ network
$\hat\theta$ with $N_0\le\poly(s,d,1/\delta)^{O(L)}$, whose weights have size
$W_{\max}\le\poly(s,d,1/\delta)^{O(L)}$ with probability $1-o(1)$ (hidden raw weights
$\le\poly/\sqrt{N_0}$ per entry, fresh noise $\le\poly$ per entry, readout capped at $\tau$), such
that
\[
  \Pp\Big(\sup_{x\in\B^d}|f_{\hat\theta}(x)-f_\infty(x)|\le\delta\Big)\ge1-o(1).
\]
Choosing $\delta\le\Delta/\poly$ gives population MSE $\loss(\hat\theta)\le\eps_+$ and
$N_0\le\poly(s,d,1/\Delta)^{O(L)}$, which is Theorem~\ref{thm:nutsandbolts}(a).
\end{theorem}

\begin{proof}
Set $K=s$ and iterate \eqref{eq:recursion} across the $L-1$ hidden layers, adding the readout (a
subsample average of the capped $\alpha_i$ against the reconstructed features, with the same lazy-swap
and Hoeffding clauses and the $o(1)$ capping tail of Proposition~\ref{prop:reg}(a)). With
$\lambda_0=(\delta/\poly)^2$, $N_0=\poly(s,d,1/\delta)^{O(L)}$ this gives
$\sup_x|f_{\hat\theta}-f_\infty|\le\delta$ whp; the weight bounds hold on the same events; and since
$f_\infty$ is MSE-$\eps$-accurate, $\loss(\hat\theta)\le(\sqrt\eps+\delta)^2\le\eps_+$ once
$\delta\le\Delta/\poly$. The construction, the per-step bound (concentration of the subsampled
features, the conditioned-chaos lazy covariance, and the lazy-swap deployment), and the
bounded-differences accounting are carried out in full in Appendix~\ref{app:comp}
(Theorem~\ref{thm:compress-app}).
\end{proof}

\section{Conservation by cloning}\label{sec:cloning}

Cloning is the cheap converse to compression. A narrow teacher's function is reproduced by a
single-site law that tilts only the teacher's own (polynomial-dimensional) active directions and
leaves every complementary direction at the prior, so the reduced entropy it costs is polynomial in
the teacher width --- the data occupy a $\poly(\mathsf W)$-dimensional subspace, and the rest is prior.

\begin{theorem}[Conservation; proof of Theorem~\ref{thm:nutsandbolts}(b)]\label{thm:cloning}
Let $f_{\theta^\star}$ be a width-$\mathsf W$ network with population MSE $\loss(\theta^\star)\le\eps_-$ and
\emph{weights of polynomial size}, $\norm{W^\star_{\ell,i}}\le\poly(\mathsf W,d)$ (so the teacher fields have
$\cm{\zeta^\star_{\ell,i}}\le\poly(\mathsf W,d)$). Then for all sufficiently large $N$,
\[
  s_N(y,\eps)\le\poly(\mathsf W,d),\qquad\text{hence}\qquad s_\infty(y,\eps)\le\poly(\mathsf W,d).
\]
In particular, since every $y\colon\B^d\to[-1,1]$ is computed to any fixed accuracy by a width-$O(2^d)$ network
with $\poly(2^d)$ weights (Lemma~\ref{lem:memorize}), $s_\infty(y,\eps)\le\poly(2^d)$ for all $y$.
\end{theorem}

\begin{proof}[Proof sketch; full argument in Appendix~\ref{app:clone}]
Write $N=\mathsf W m$ and partition the $N$ neurons at each layer into $\mathsf W$ \emph{clone groups} of size $m$,
group $b$ cloning teacher neuron $b$. The exact clone is the block weight
\[
  W^{\ell+1}=\tfrac1{\sqrt m}\bigl(W^{\star,\ell+1}\otimes\mathbf 1_{m\times m}\bigr)
  =\tfrac1{\sqrt m}\begin{pmatrix}
  W^\star_{11}\mathbf 1 & \cdots & W^\star_{1\mathsf W}\mathbf 1\\
  \vdots & & \vdots\\
  W^\star_{\mathsf W1}\mathbf 1 & \cdots & W^\star_{\mathsf W\mathsf W}\mathbf 1\end{pmatrix},
\]
with first-layer weights and readouts cloned likewise ($W^1_j=W^{\star,1}_b$, $\alpha_j=\alpha^\star_b$
for $j$ in group $b$). On a group-constant previous layer this reproduces the teacher exactly: the
$\mu$P sums telescope, $\frac1{\sqrt N}\sum_iW^{\ell+1}_{ji}z_i=\frac1{\sqrt{\mathsf W}}\sum_aW^\star_{ba}z^\star_a
=\zeta^{\star,\ell+1}_b$, and the $\tfrac1N$ readout average becomes the teacher's $\tfrac1{\mathsf W}$ average.
The block factor $\tfrac1{\sqrt m}$ is the only rescaling: it reconciles the wide net's $\mu$P hidden
normalization $\tfrac1{\sqrt N}$ with the teacher's $\tfrac1{\sqrt{\mathsf W}}$, since
$\sqrt N=\sqrt{\mathsf W m}$ (the readout alone is intensive, $\tfrac1N\to\tfrac1{\mathsf W}$), exactly
the convention of \eqref{eq:forward}.

\emph{Trial law (a microscopic cylinder).} Let $\pi^{\mathrm{cyl}}_N$ be the prior conditioned on the
cylinder in which, for each neuron, the $\mathsf W$ group-mean coordinates
$\ip{W^{\ell+1}_{j,\cdot}}{u_a}$ ($u_a=m^{-1/2}\mathbf 1_{\mathrm{group}\,a}$), the $d$ first-layer
weights, and the readout lie within radius $\rho$ of their cloned targets; the remaining $N-\mathsf W$
orthogonal coordinates of each row are left at the prior. The cylinder has \emph{positive} radius
precisely to keep $-\log\pi^0_N(\mathrm{cyl})$ finite.

\emph{The free directions are harmless --- the data is $\mathsf W$-dimensional.} The constrained group means
span a $\mathsf W$-dimensional subspace of each row; on a group-constant input the free orthogonal weights
contribute nothing, since $\sum_{i\in a}W^\perp_{ji}=0$ annihilates the group mean. Acting instead on
the within-group \emph{deviation} $z^\ell_i-\bar z^\ell_a$, they inject noise of size the previous
spread. The first-layer tube gives spread $\sigma_1\le L_\phi\rho\sqrt d$, and one layer multiplies it
by $O(L_\phi)$ (the free weights have $O(1)$ operator norm), so $\sigma_\ell\le\poly(d,L_\phi)^{L}\rho$;
the group means then track the teacher to $\poly(\mathsf W,d,B,L_\phi)^L\rho$ (Lemma~\ref{lem:spread}).
The clone therefore reproduces the teacher \emph{uniformly},
$\sup_x|f_{\mathrm{clone}}(x)-f_{\theta^\star}(x)|\le\poly(\mathsf W,d,B)^L\rho$, in both empirical and
population MSE.

\emph{Entropy.} Only the $d+\mathsf W+1$ constrained coordinates per neuron cost relative entropy; each is a
Gaussian coordinate pinned to a tube of radius $\rho$ about a target of size $\le\poly(\mathsf W,d)$, costing
$\log\tfrac1\rho+O(\poly(\mathsf W,d))$, while the $N-\mathsf W$ free directions cost zero. As $\pi^{\mathrm{cyl}}_N$
is the prior restricted to this product cylinder,
\[
  \KL\bigl(\pi^{\mathrm{cyl}}_N\Vert\pi^0_N\bigr)=-\log\pi^0_N(\mathrm{cyl})
  =N\bigl[(d+\mathsf W+1)\log\tfrac1\rho+O(\poly(\mathsf W,d))\bigr].
\]
Taking $\rho=\exp(-\poly(\mathsf W,d,B))$ so that $\poly^L\rho\le\Delta$ gives cloned loss
$\le\eps_-+\Delta\le\eps$ with probability $\ge\tfrac12$ over the free coordinates, and
$K_{\mathrm{clone}}:=\tfrac1N\KL(\pi^{\mathrm{cyl}}_N\Vert\pi^0_N)=\poly(\mathsf W,d,B)$.

\emph{Change of measure.} With $\pi^{\mathrm{cyl}}_N(\loss\le\eps)\ge\tfrac12$, the data-processing
inequality on $\{\loss\le\eps\}$ gives
$-\log\pi^0_N(\loss\le\eps)\le 2\KL(\pi^{\mathrm{cyl}}_N\Vert\pi^0_N)+O(1)=N\poly(\mathsf W,d,B)$, so
$s_N(y,\eps)\le\poly(\mathsf W,d,B)$ for all large $N$ and the $\limsup$ is the same. The identical
construction with $\widehat{\loss}_D$ in place of $L$ yields the empirical statement
$\pi^0_N(\widehat{\loss}_D\le\widehat{\loss}_D(\theta^\star)+\eta)\ge e^{-N\poly(\mathsf W,d,B,\log1/\eta)}$ used in the
forward assembly. Finally Lemma~\ref{lem:memorize} realizes any $y$ to any fixed accuracy at width
$\mathsf W=O(2^d)$ with $\poly(2^d)$ weights, giving the last claim.
\end{proof}

\noindent Together with Theorem~\ref{thm:A}, the reduced entropy and the minimal
polynomial-weight teacher width are polynomially equivalent: $s_\infty(y,\eps_-)\le s$ yields a
teacher of width $\poly(s)$ with $\poly(s)$ weights (Theorem~\ref{thm:A}), and a
polynomial-weight teacher of width $\mathsf W$ yields $s_\infty\le\poly(\mathsf W)$ (Theorem~\ref{thm:cloning}).
This is the conservation of the reduced entropy across widths that drives the equivalence of
Theorem~\ref{thm:headline}.

\section{Learnability}\label{sec:learnability}

\subsection{Low reduced entropy implies Gibbs learnability}\label{sec:gen}

Fix the training model: $n$ i.i.d.\ inputs $x_1,\dots,x_n\sim\mu$ with labels $y(x_i)$ and
empirical loss $\widehat{\loss}_D(\theta)=\frac1n\sum_i(f_\theta(x_i)-y(x_i))^2$. The learner in this
section is not the hard empirical ball.  It is the adaptive empirical Gibbs posterior
\begin{equation}\label{eq:emp-gibbs}
  \pi^D_{N,\bar\beta}(d\theta)
  \;=\;\frac{\exp(-N\bar\beta\,\widehat{\loss}_D(\theta))}{Z_D(\bar\beta)}\,\pi^0_N(d\theta),
\end{equation}
with intensive inverse temperature $\bar\beta$.  The hard population ball remains the definition of
reduced entropy; it is used only to produce polynomial KL/free-energy budgets.  The smooth Gibbs
posterior is the object that has a differentiable Boltzmann score and hence a mean-field tower.

The forward direction rests on placing the learned function in a \emph{fixed}, data-independent
class and applying uniform convergence there. The following isolates that step; its only inputs are
the weight bounds the compression of \S\ref{sec:compression} guarantees, and its covering number is
independent of $2^d$.

\begin{lemma}[Uniform convergence on a bounded-weight class]\label{lem:unifconv}
Fix depth $L$ and let $\mathcal H_{N_0,B}$ be the class of width-$N_0$ networks \eqref{eq:forward}
with weights bounded as in the construction (hidden entries $\le B/\sqrt{N_0}$, fresh-noise entries
$\le B$, readouts $|\alpha_i|\le B$). Then:
\begin{enumerate}[label=\textup{(\alph*)},nosep,topsep=3pt]
\item every $f\in\mathcal H_{N_0,B}$ has $\norm{f}_\infty\le B$, and
  $\ell_f(x)=(f(x)-y(x))^2$ lies in $[0,(B+1)^2]$ and is $2(B+1)$-Lipschitz in
  $\norm{f}_\infty$;
\item $\log\mathcal N(\mathcal H_{N_0,B},\delta,\norm{\cdot}_\infty)
  \le P\log(3\Lambda B/\delta)$, with
  $P=O(N_0d+LN_0^2)$ parameters and parameter-to-function sup-norm Lipschitz constant
  $\Lambda=\poly(B,d)^{O(L)}$; in particular it is
  $\le\poly(N_0,B,d)^{O(L)}\log(1/\delta)$, \emph{independent of $2^d$};
\item for $n\ge\poly(N_0,B,d,1/\Delta)^{O(L)}\log(1/\eta)$ i.i.d.\ samples,
  $\Pp\big(\sup_{f\in\mathcal H_{N_0,B}}|\widehat{\loss}_D(f)-\loss(f)|>\Delta/4\big)\le\eta$.
\end{enumerate}
\end{lemma}

\begin{proof}
(a) $|f_\theta(x)|\le\frac1{N_0}\sum_i|\alpha_i|\,|\phi|\le B$; with $|y|\le1$ this gives the range,
and $|\ell_f-\ell_g|=|f-g|\,|f+g-2y|\le 2(B+1)\norm{f-g}_\infty$.
(b) The map $\theta\mapsto f_\theta$ is sup-norm Lipschitz with constant
$\Lambda=\poly(B,d)^{O(L)}$: each layer composes the $1$-Lipschitz $\phi$ with a weight multiply of
operator norm $\le\poly(B,d)$ over the $\mu$P/$1/N_0$ normalizations, and the bound multiplies over
$L$ layers.  Gridding the $P=O(N_0d+LN_0^2)$ parameters, each in $[-B,B]$, at scale
$\delta/\Lambda$ yields a sup-norm $\delta$-cover of size $(3\Lambda B/\delta)^P$.
(c) For fixed $f$ the losses are i.i.d.\ in $[0,(B+1)^2]$, so Hoeffding gives
$\Pp(|\widehat{\loss}_D(f)-\loss(f)|>t)\le2e^{-2nt^2/(B+1)^4}$. Take a sup-norm $\delta$-cover with
$\delta=\Delta/(16(B+1))$ and $t=\Delta/8$; a union bound over the cover gives the claim.
\end{proof}

\begin{theorem}[Low entropy $\Rightarrow$ empirical Gibbs learnability]\label{thm:gen}
Suppose $s_\infty(y,\eps_-)\le s$ and $\phi$ satisfies the standing assumption, with fixed depth $L$.
Then there are
\[
  n,\;N_0,\;\bar\beta \;\le\; \poly(s,d,1/\Delta)^{O(L)}
\]
such that the infinite-width empirical Gibbs learner satisfies
\[
  \E_D\,\E_{\theta\sim\pi^D_{\infty,\bar\beta}}\loss(\theta)\le\eps_+,
\]
and the same conclusion holds for the width-$N_0$ Gibbs posterior
$\pi^D_{N_0,\bar\beta}$, up to an arbitrarily small $o(1)$ slack absorbed into $\eps_+$ by increasing
the polynomial constants.
\end{theorem}

\begin{proof}[Proof]
\emph{Step 0: the learner is the smooth empirical Gibbs object.}  For fixed $D$ and $\bar\beta$,
\eqref{eq:emp-gibbs} is a finite-data mean-field Gibbs measure.  Appendix~\ref{app:mf},
Lemma~\ref{lem:mf}, gives the $N\to\infty$ limit as a mixture over empirical mean-field vacua.  We
condition on a vacuum $\mathfrak v$; then there is a deterministic single-site tower
$P^{D,\mathfrak v}_\ell$, propagation of chaos holds conditionally, and the Boltzmann tilt has the
empirical-loss score.  The limiting learned function in this vacuum is
$f^{D,\mathfrak v}(x)=\E_{P^{D,\mathfrak v}_{L-1}}[\alpha\phi(\zeta(x))]$.  The regularity and
truncation package of Appendix~\ref{app:reg} applies uniformly to every vacuum with the same
KL-rate bound, so after proving the estimates conditionally we integrate over the vacuum mixture.

\emph{Step 1: the empirical free-energy witness is the population-good teacher.}  Set
$\gamma=\Delta/32$.  First use the reduced-entropy hypothesis at the inner radius to form a
population Gibbs witness at the intermediate accuracy $\eps_-+\gamma/4$: by the multiplier bound,
its intensive inverse temperature and KL rate are $\poly(s,d,1/\Delta)$.  In the population
mean-field decomposition, condition on a component/tower whose loss and free-energy budget satisfy
the same bounds up to harmless constant factors; equivalently, apply the following compression
argument to the Gibbs witness and then fix a good realization.  Applying the quantitative
pointwise compression theorem (Theorem~\ref{thm:A}, with sup-norm accuracy
$\delta\le \gamma/\poly$) to this fixed population-good tower produces a \emph{data-independent} polynomial
teacher $f_\star\in\mathcal H_{N_\star,B_\star}$ such that
\[
  N_\star,B_\star\le \poly(s,d,1/\Delta)^{O(L)},
  \qquad
  \loss(f_\star)\le \eps_-+\gamma .
\]
Fix one realization of this teacher on the high-probability compression event.  We choose the final
sample size large enough that uniform convergence also holds on the teacher class
$\mathcal H_{N_\star,B_\star}$; then, except on an $o(1)$ event in $D$,
\[
  \widehat{\loss}_D(f_\star)\le \eps_-+2\gamma .
\]
Now clone and thicken this fixed teacher, not a memorizer of the sample.  The empirical version of
Theorem~\ref{thm:cloning} applied to $\widehat{\loss}_D$ gives, for every such dataset and all
sufficiently large widths $N$,
\[
  \pi^0_N(\widehat{\loss}_D\le\eps_-+3\gamma)
  \ge e^{-N K_\star},
  \qquad
  K_\star\le \poly(s,d,1/\Delta)^{O(L)} .
\]
The empirical Gibbs variational bound (Lemma~\ref{lem:multiplier-app}) therefore yields
\begin{equation}\label{eq:emp-free-energy}
  \E_{\pi^D_{N,\bar\beta}}\widehat{\loss}_D
  \le \eps_-+3\gamma+K_\star/\bar\beta,
  \qquad
  \tfrac1N\KL(\pi^D_{N,\bar\beta}\Vert\pi^0_N)
  \le K_\star+\bar\beta(\eps_-+3\gamma).
\end{equation}
Choose $\bar\beta\ge K_\star/\gamma$, still bounded by
$\poly(s,d,1/\Delta)^{O(L)}$, so
\[
  \E_{\pi^D_{N,\bar\beta}}\widehat{\loss}_D\le\eps_-+4\gamma,
  \qquad
  \tfrac1N\KL(\pi^D_{N,\bar\beta}\Vert\pi^0_N)
  \le R_\star:=\poly(s,d,1/\Delta)^{O(L)} .
\]
Thus the empirical Gibbs posterior has both low empirical loss and a polynomial KL rate for the
right reason: the witness is the low-entropy population teacher supplied by
$s_\infty(y,\eps_-)\le s$.

\emph{Step 2: truncate and compress into a fixed class.}  The KL-rate bound in
\eqref{eq:emp-free-energy} feeds Proposition~\ref{prop:package} and the explicit kept-set version,
Proposition~\ref{prop:trunc}.  After discarding a posterior subpopulation whose suffix-output
contribution is at most $\eta$ in sup norm, the empirical Gibbs tower has polynomial readout caps,
operator norms, active Cameron--Martin coefficients, and empirical lazy-score bounds.  The empirical
active block is taken basis-free as
\[
  U_\ell=U^{\rm pop}_\ell\oplus U^{\rm emp}_\ell,
\]
where $U^{\rm pop}_\ell$ is the population spectral top space and $U^{\rm emp}_\ell$ is the top
space of the normalized empirical evaluation covariance
$\frac1n\sum_{a\le n}(K_\ell^{1/2}e_{x_a})\otimes(K_\ell^{1/2}e_{x_a})$ above cutoff $\tau$.  Thus
$\dim U_\ell\le\lambda_0^{-1}+\tau^{-1}$, and the empirical lazy leakage is not asserted to vanish;
Appendix~\ref{app:comp}, Lemma~\ref{lem:softgram}, gives instead
\[
  c^\top \widehat C^\ell_D c\le \tau B^2
\]
for every empirical-loss score $\sum_a c_a e_{x_a}$ with $\norm c^2\le B^2/n$.  Taking
$\tau,\lambda_0$ inverse-polynomial makes the population and empirical lazy-swap errors
$\le\eta$ while keeping $\dim U_\ell$ polynomial.

The pointwise compression theorem (Theorem~\ref{thm:A}, in the empirical version supplied by
Appendix~\ref{app:comp}) then produces, conditionally on each vacuum, a data-dependent proxy $h_{D,\mathfrak v}\in\mathcal H_{N_0,B}$ with
$N_0,B\le\poly(s,d,\bar\beta,1/\Delta)^{O(L)}=\poly(s,d,1/\Delta)^{O(L)}$ and
\begin{equation}\label{eq:proxy-close}
  \E_{\mathfrak v\sim\rho_D}\E_{\pi^D_{\infty,\bar\beta}(\cdot\mid\mathfrak v)}\norm{f_\theta-h_{D,\mathfrak v}}_\infty\le\eta .
\end{equation}
Here $\mathcal H_{N_0,B}$ itself is fixed independently of $D$; only the point $h_{D,\mathfrak v}$ in the class is
random through the sample and the vacuum.

\emph{Step 3: uniform convergence transfers empirical Gibbs loss to population loss.}  Enlarge
$N_0,B$ if necessary so that $\mathcal H_{N_\star,B_\star}\subseteq\mathcal H_{N_0,B}$.  Now choose
$n\ge\poly(N_0,B,d,1/\Delta)^{O(L)}=\poly(s,d,1/\Delta)^{O(L)}$ so that
Lemma~\ref{lem:unifconv} gives, with probability $1-o(1)$ over $D$,
\begin{equation}\label{eq:uc-event}
  \sup_{h\in\mathcal H_{N_0,B}}|\widehat{\loss}_D(h)-\loss(h)|\le\Delta/4 .
\end{equation}
On this event, Lipschitz transfer of the squared loss and \eqref{eq:proxy-close} first compare the
Gibbs function to its compressed proxy,
\[
  \E_{\pi^D_{\infty,\bar\beta}}\loss(\theta)
  \le
  \E_{\mathfrak v\sim\rho_D}\loss(h_{D,\mathfrak v})+\poly(B)\eta,
  \qquad
  \E_{\mathfrak v\sim\rho_D}\widehat{\loss}_D(h_{D,\mathfrak v})
  \le
  \E_{\pi^D_{\infty,\bar\beta}}\widehat{\loss}_D(\theta)+\poly(B)\eta .
\]
Since each $h_{D,\mathfrak v}$ lies in the fixed class $\mathcal H_{N_0,B}$, uniform convergence
\eqref{eq:uc-event} applies to the proxy rather than to the raw Gibbs draw.  Thus
\[
  \E_{\pi^D_{\infty,\bar\beta}}\loss(\theta)
  \le
  \E_{\pi^D_{\infty,\bar\beta}}\widehat{\loss}_D(\theta)
  +\Delta/4+\poly(B)\eta
  \le \eps_-+4\gamma+\Delta/4+\poly(B)\eta+o(1)
  \le\eps_+,
\]
for $\eta\le\Delta/\poly(B)$, since $\eps_+-\eps_-\ge\Delta$ after shrinking constants in the
intermediate gaps.  On the exceptional $o(1)$ fraction of datasets, Proposition~\ref{prop:trunc}
gives uniform integrability of the loss under the Gibbs law, so their contribution is $o(1)$.
This proves the infinite-width claim.

\emph{Step 4: finite width.}  The same argument applies to the width-$N_0$ Gibbs posterior after the
kept-set truncation of Proposition~\ref{prop:trunc}.  The empirical free-energy bound
\eqref{eq:emp-free-energy} is valid at width $N_0$ once $N_0$ is chosen large enough to contain and
clone the fixed population teacher $f_\star$.  Truncation changes expected outputs, empirical loss,
and population loss by at most
$\eta$; the truncated posterior is supported on $\mathcal H_{N_0,B}$, so \eqref{eq:uc-event} applies
directly to its draws.  Letting $\eta=o(\Delta)$ gives
$\E_D\E_{\pi^D_{N_0,\bar\beta}}\loss\le\eps_+$.
\end{proof}

\noindent The pointwise compression is exactly what lets a draw from the a priori width-$N$ (or
infinite-width) Gibbs posterior be compared to a \emph{fixed} finite-dimensional class for the union
bound.  The hard reduced-entropy ball supplies the polynomial budget; it is not itself the learner.

\subsection{The converse: Gibbs learnability implies low reduced entropy}\label{sec:conv}

\begin{theorem}[Gibbs learnability $\Rightarrow$ low entropy]\label{thm:conv}
Suppose that for some $n$ and some polynomially bounded intensive temperature
$\bar\beta\le\poly(n,d)$, the empirical Gibbs learner satisfies
\[
  \E_D\E_{\theta\sim\pi^D_{N,\bar\beta}}\loss(\theta)\le\eps_-
\]
for all sufficiently large $N$ (or in the infinite-width limit).  If $\eps_-<\eps$ with a fixed gap,
then $s_\infty(y,\eps)\le\poly(n,d)$, with polynomial dependence also on $\bar\beta$ and the gap.
\end{theorem}

\begin{proof}
For every dataset $D$, memorization plus cloning gives an empirical prior-mass lower bound
\[
  \pi^0_N(\widehat{\loss}_D\le\eps_-+o(1))\ge e^{-NK_D},
  \qquad K_D\le\poly(n,d).
\]
The empirical free-energy bound of Lemma~\ref{lem:multiplier-app} therefore gives
\begin{equation}\label{eq:gibbs-conv-kl}
  \tfrac1N\KL(\pi^D_{N,\bar\beta}\Vert\pi^0_N)
  \le K_D+\bar\beta(\eps_-+o(1))\le\poly(n,d,\bar\beta)
\end{equation}
for every $D$.

Choose an intermediate threshold $t\in(\eps_-,\eps)$.  By the assumed population performance and
Markov's inequality over the draw of $D$, the set of datasets for which
\[
  \E_{\pi^D_{N,\bar\beta}}\loss\le t_0:=\frac{\eps_-+t}{2}
\]
has probability at least $1-\eps_-/t_0>0$.  Fix such a dataset.  Markov's inequality under the Gibbs
law then gives
\[
  \pi^D_{N,\bar\beta}(\loss\le t)
  \ge 1-t_0/t=:q>0,
\]
where $q$ depends only on the fixed accuracy gap.  Apply data processing for the binary partition
$A=\{\loss\le t\}$ to $Q=\pi^D_{N,\bar\beta}$ and $P=\pi^0_N$: if $Q(A)\ge q$, then
$P(A)\ge\exp(-q^{-1}\KL(Q\Vert P)-O_q(1))$.  Together with \eqref{eq:gibbs-conv-kl},
\[
  \pi^0_N(\loss\le t)
  \ge \exp\{-N\poly(n,d,\bar\beta,1/q)-O_q(1)\}.
\]
Since $t<\eps$, monotonicity gives the same lower bound for $\pi^0_N(\loss\le\eps)$.  Hence
$s_N(y,\eps)\le\poly(n,d,\bar\beta,1/q)+o(1)$ and the limsup gives
$s_\infty(y,\eps)\le\poly(n,d)$ for polynomial $\bar\beta$, with the fixed-gap constants absorbed
into the polynomial.
\end{proof}

\noindent The mechanism is still memorization plus change of measure, but the change of measure now
starts from the empirical Gibbs law rather than a hard empirical posterior.

\subsection{Assembly of the equivalence}\label{sec:assembly}

\begin{proof}[Proof of Theorem~\ref{thm:headline}]
Throughout, the accuracy levels $\eps_-<\eps<\eps_+$ and $1/\Delta$ are baseline constants, absorbed
into the polynomials; ``$\loss=o(1)$'' is read at a fixed accuracy and recovered by the usual rounding
of a boolean target. We show that (i) and (ii) are each equivalent to $s_\infty(y_d,\eps)$ being
polynomial in $d$, which also gives (iii) and the equivalence of (i) and (ii).

\smallskip\noindent\emph{$s_\infty$ polynomial $\Rightarrow$ (i) and (ii).} By
Theorem~\ref{thm:gen} with $s=\poly(d)$, both the infinite-width empirical Gibbs posterior and the
width-$N_0$ empirical Gibbs posterior ($N_0=\poly(d)$) trained on $n=\poly(d)$ samples have population
MSE $\le\eps_+=o(1)$ in expectation, for polynomial intensive temperature $\bar\beta$.

\smallskip\noindent\emph{(i) or (ii) $\Rightarrow s_\infty$ polynomial.} Gibbs learnability from
$n=\poly(d)$ samples at polynomial intensive temperature gives, by Theorem~\ref{thm:conv},
$s_\infty(y_d,\eps)\le\poly(n,d)=\poly(d)$.

\smallskip\noindent\emph{Equivalence with (iii) and the reduced-entropy characterization.} If
$s_\infty(y_d,\eps_-)\le\poly(d)$ then Theorem~\ref{thm:A} produces a network of width $\poly(d)$
whose weights have polynomial size whp (hidden raw weights $\le\poly/\sqrt{N_0}$ per entry, fresh
noise $\le\poly$, readouts capped at $\tau=\poly$), with $\loss\le\eps_+=o(1)$, which is (iii);
conversely a width-$W$ polynomial-weight network with $\loss\le\eps_-$ gives
$s_\infty\le\poly(W)$ by Theorem~\ref{thm:cloning}, so (iii) with $W=\poly(d)$ returns
$s_\infty\le\poly(d)$. Hence all of (i), (ii), (iii) hold exactly when $s_\infty(y_d,\eps)$ is
polynomial in $d$.
\end{proof}

\noindent The pointwise compression (Theorem~\ref{thm:A}) is the engine throughout: it supplies the
fixed low-complexity class for the forward generalization (Theorem~\ref{thm:gen}) and the explicit
polynomial-width teacher in the equivalence with (iii), and --- via the $\sup_x$ agreement of
Theorem~\ref{thm:nutsandbolts}(a) --- it is what makes the infinite-width and polynomial-width
learners agree as functions, not merely in risk.

\section{Discussion}\label{sec:discussion}

\paragraph{Why $\gamma=1$ is the critical case.}
The mean-field exponent $\gamma=1$ is the delicate one. At the over-rich scalings $\gamma>1$ the
lazy bulk of each feature kernel contributes below the accuracy threshold, so compression may
simply discard it and the lazy-swap is unnecessary; at $\gamma=1$ the lazy content is of the same
order as $\eps$ and must be reproduced \emph{in law}, which is exactly what the lazy-swap lemma
(Lemma~\ref{lem:swap}) provides. This is why the $\gamma=1$ result is not a corollary of the
$\gamma>1$ one, and why the pointwise (rather than mean-square) reproduction is forced here. The
scalings $\gamma>1$, where the argument is softer, are treated alongside the general finite-width
theory in a companion paper.

\paragraph{Finite width.}
This paper establishes the equivalence between infinite width and \emph{polynomial} width. The only
place infinite width is genuinely used is propagation of chaos (Definition~\ref{def:chaos}), which
makes the subsampled feature averages exact single-site expectations; polynomial width needs no such
input because the network is already small and ordinary uniform convergence applies, and cloning
transfers in the other direction without it. At general (super-polynomial) finite width these
averages fluctuate, and controlling them uniformly over the $2^d$ inputs requires a concentration
hypothesis on the single-site law. Establishing that hypothesis is the central analytic task of the
companion paper; as $N\to\infty$ it degenerates into chaos and the present results are recovered.

\paragraph{Relation to circuit complexity.}
A polynomial-width network is, up to bounded precision, a polynomial-size circuit, so the
characterization reads morally as ``learnable if and only if computed by a small circuit'' --- a
$\mathsf{P}/\mathrm{poly}$-flavored statement for mean-field learning. We resist stating it as a
theorem: our reductions cost a factor $O(L)$ in the exponent and do not preserve depth, so the
induced complexity statement is coarse, tracking polynomial size but not the depth that
distinguishes circuit classes. We therefore offer the circuit reading as motivation for the
reduced-entropy measure rather than as a complexity-theoretic equivalence.

\paragraph{What the pointwise witness buys.}
The reduced entropy is defined through mean-squared error, which is the robust complexity measure;
yet the compression that certifies it is pointwise. This is not a cosmetic strengthening. It is the
pointwise agreement that lets a draw from the (a priori width-$N$, even infinite) posterior be
placed in a fixed finite-dimensional class for a union bound, and that makes the infinite-width and
polynomial-width learners agree \emph{as functions} rather than merely in risk --- the sense in
which the inductive bias, and not just the achievable loss, is conserved across widths.

\section*{Acknowledgments}

We thank Lauren Greenspan, Lawrence Chan, Zohar Ringel, Noa Rubin, Jake Mendel, and others for helpful discussions.

The authors used large language model tools to assist with drafting and editing parts of the paper, including portions of the exposition and lemma statements. The mathematical arguments, references, and final manuscript were reviewed by the authors, who take full responsibility for the content.

The first author was supported by Princint (Principles of Intelligence) through funding from grants from the AI Security Institute and Coefficient Giving.

\appendix

\paragraph{Notation in the appendices.}
The appendices develop the construction in a slightly more general, induction-friendly notation than
the body; the dictionary is as follows. The active space is written $U_\ell$ (the body's $V$ of the
population spectral split, \S\ref{sec:split}, and the enlarged $U_{\mathrm{eff}}$ of \S\ref{sec:gen}
once the empirical evaluation directions are adjoined). Layer widths are $M_\ell$ (the body's
retained width $N_0$). The intensive multiplier is $\bar\beta$ (the body's $\bar\beta_\star$), and
$R$ denotes the resulting Gibbs KL rate $\tfrac1N\KL(\pi\Vert\pi^0_N)\le K+\bar\beta\eps_-=
\poly(K,1/\Delta)$, with $K=s_\infty(y,\eps_-)$ the entropy budget. The per-layer regularity
constants are collected as $A_+$, with Lipschitz/adjoint constants $\Lambda,\Gamma$ for the suffix;
these are what the body summarizes as the susceptibility $c_\star$ and the readout cap $\tau$. All
polynomial bounds carry exponents depending only on the fixed depth $L$.

\section{Regularity from the entropy budget}\label{app:reg}

This appendix proves that a single scalar --- the intensive KL rate of the Gibbs posterior ---
controls every regularity quantity the compression argument of Appendix~\ref{app:comp} consumes:
capped readouts, normalized operator norms, Cameron--Martin active coefficients, and lazy scores.
Two principles keep the estimates clean. First, although the width-$N$ posterior is not a product,
the prior is, and the chain rule for relative entropy passes the budget to the \emph{averaged
single-site marginal} (Lemma~\ref{lem:ss-kl}); all per-neuron envelopes then follow from
finite-dimensional Gaussian estimates against that marginal. Second --- and this is the point of the
suffix-invariant formulation below --- when we bound the layer-$\ell$ score we do so \emph{through
the polynomial suffix already built to its right}, never through the raw width-$N$ matrices, so the
lazy score inherits the suffix's polynomial Lipschitz constant directly.

Throughout, $\pi$ denotes a population Gibbs posterior $\pi_{N,\bar\beta}\propto e^{-N\bar\beta\loss}\pi^0_N$
or its empirical analogue $\pi^D_{N,\bar\beta}\propto e^{-N\bar\beta\widehat{\loss}_D}\pi^0_N$, with intensive
inverse temperature $\bar\beta$. We write $\loss_\star$ for $L$ or $\widehat{\loss}_D$ as appropriate.

\subsection{The adaptive-temperature Gibbs law}

\begin{lemma}[Multiplier]\label{lem:multiplier-app}
Suppose $s_\infty(y,\eps_-)\le K$. Then for every $\bar\beta>0$,
\[
  \limsup_{N\to\infty}\E_{\pi_{N,\bar\beta}}L\;\le\;\eps_-+\frac{K}{\bar\beta},
  \qquad
  \limsup_{N\to\infty}\tfrac1N\KL\!\big(\pi_{N,\bar\beta}\,\Vert\,\pi^0_N\big)\;\le\;K+\bar\beta\eps_- .
\]
In particular, taking $\bar\beta=2K/(\eps-\eps_-)$ gives expected loss
$\le\eps_-+(\eps-\eps_-)/2<\eps$ and KL rate $\le K\bigl(1+2\eps_-/(\eps-\eps_-)\bigr)$, which is
$\poly(K,1/\Delta)$ since $\eps-\eps_-\ge\Delta$. The same bounds hold for the empirical
posterior with $L\to\widehat{\loss}_D$, $\eps_-\to\tau$, $K\to K_D$ whenever
$\pi^0_N(\widehat{\loss}_D\le\tau)\ge e^{-N(K_D+o(1))}$.
\end{lemma}

\begin{proof}
Write $Z_N(\bar\beta)=\E_{\pi^0_N}e^{-N\bar\beta\loss}$ and $\Phi_N(\bar\beta)=-\tfrac1N\log Z_N(\bar\beta)$.
Because $s_\infty(y,\eps_-)\le K$, for every $\delta>0$ and all large $N$ we have
$\pi^0_N(\loss\le\eps_-)\ge e^{-N(K+\delta)}$, hence
\[
  Z_N(\bar\beta)\;\ge\;\int_{\{\loss\le\eps_-\}}e^{-N\bar\beta\loss}\,d\pi^0_N
  \;\ge\;e^{-N\bar\beta\eps_-}\,\pi^0_N(\loss\le\eps_-)
  \;\ge\;e^{-N(\bar\beta\eps_-+K+\delta)} ,
\]
so $\Phi_N(\bar\beta)\le \bar\beta\eps_-+K+\delta$. By the definition of $\pi_{N,\bar\beta}$,
\[
  \tfrac1N\KL\!\big(\pi_{N,\bar\beta}\Vert\pi^0_N\big)
  =\E_{\pi_{N,\bar\beta}}\!\Big[\tfrac1N\log\tfrac{d\pi_{N,\bar\beta}}{d\pi^0_N}\Big]
  =\Phi_N(\bar\beta)-\bar\beta\,\E_{\pi_{N,\bar\beta}}L\;\ge\;0 .
\]
The first inequality gives $\E_{\pi_{N,\bar\beta}}\loss\le \Phi_N(\bar\beta)/\bar\beta\le \eps_-+(K+\delta)/\bar\beta$;
the second, together with $\E_{\pi_{N,\bar\beta}}\loss\ge0$, gives
$\tfrac1N\KL\le\Phi_N(\bar\beta)\le K+\bar\beta\eps_-+\delta$. Let $\delta\downarrow0$. The empirical
statement is identical, replacing $L$ by $\widehat{\loss}_D$ and the entropy lower bound by its empirical
counterpart.
\end{proof}

\subsection{Entropy consequences}

\begin{lemma}[Single-site KL decomposition]\label{lem:ss-kl}
Let $\pi$ be any (exchangeable) measure on $\theta=(\theta_1,\dots,\theta_N)$ and let the prior
factorize, $\pi^0=\bigotimes_{i=1}^N\pi^0_{(i)}$. Let $\pi_{(i)}$ be the $i$th marginal of $\pi$ and
$\bar\rho=\tfrac1N\sum_i\pi_{(i)}$ the averaged single-site marginal (against the common
single-site prior $\rho^0$). Then
\[
  \KL\!\big(\bar\rho\,\Vert\,\rho^0\big)
  \;\le\;\tfrac1N\sum_{i=1}^N\KL\!\big(\pi_{(i)}\,\Vert\,\rho^0\big)
  \;\le\;\tfrac1N\KL\!\big(\pi\,\Vert\,\pi^0\big) .
\]
In particular, if $\tfrac1N\KL(\pi\Vert\pi^0)\le R$ then $\KL(\bar\rho\Vert\rho^0)\le R$.
\end{lemma}

\begin{proof}
The first inequality is convexity of relative entropy in its first argument. For the second, the
chain-rule identity for a product reference measure reads
\[
  \KL\!\big(\pi\,\Vert\,\textstyle\bigotimes_i\pi^0_{(i)}\big)
  =\KL\!\big(\pi\,\Vert\,\textstyle\bigotimes_i\pi_{(i)}\big)+\sum_{i=1}^N\KL\!\big(\pi_{(i)}\Vert\pi^0_{(i)}\big)
  \;\ge\;\sum_{i=1}^N\KL\!\big(\pi_{(i)}\Vert\rho^0\big),
\]
since the total-correlation term is nonnegative and $\pi^0_{(i)}=\rho^0$. Divide by $N$.
\end{proof}

\begin{lemma}[Gaussian moment and projection bounds]\label{lem:gauss}
Let $\gamma=N(0,I_m)$ and let $q\ll\gamma$ with $\KL(q\Vert\gamma)\le R$. Then:
\begin{enumerate}[label=\textup{(\alph*)},nosep]
\item for any unit vector $a$, $\ \E_q\ip{X}{a}^2\le 4R+2\log 2$;
\item for any subspace $U$ of dimension $r$, $\ \E_q\norm{P_U X}^2\le 4R+2r\log 2$;
\item for any $T>0$, $\ \E_q\big[\norm{P_U X}\,\mathbf 1_{\norm{P_U X}>T}\big]\le (4R+2r\log2)/T$;
\item if $X=(X_x)_{x\in\B^d}$ is any centered Gaussian field with $\sup_x\E X_x^2\le1$, then
  $\E_q\norm{X}_\infty^2\le 4R+4d\log2+2\log2$.
\end{enumerate}
The first three bounds hold for $\gamma=N(0,K)$ with $\norm{\,\cdot\,}$ read in Cameron--Martin
coordinates $X\mapsto K^{-1/2}X$; the last bound is in the original field coordinates and uses only
$K(x,x)\le1$.
\end{lemma}

\begin{proof}
The Donsker--Varadhan inequality $\E_q F\le \KL(q\Vert\gamma)+\log\E_\gamma e^{F}$ applied to
$F=\tfrac14\ip{X}{a}^2$ gives, using $\E_\gamma e^{tZ^2}=(1-2t)^{-1/2}$ at $t=\tfrac14$,
$\tfrac14\E_q\ip{X}{a}^2\le R+\tfrac12\log2$, which is (a). For (b) take $F=\tfrac14\norm{P_UX}^2$
and $\E_\gamma e^{t\chi^2_r}=(1-2t)^{-r/2}$. Part (c) is $\E[Y\mathbf 1_{Y>T}]\le\E[Y^2]/T$ with
$Y=\norm{P_UX}$. Whitening reduces the $N(0,K)$ case to the standard one. For (d), by the union
bound and the one-dimensional Gaussian moment generating function,
\[
  \E_\gamma \exp\{\tfrac14\norm X_\infty^2\}
  \le \sum_{x\in\B^d}\E_\gamma e^{X_x^2/4}
  \le 2^d\sqrt2 .
\]
Applying Donsker--Varadhan to $F=\tfrac14\norm X_\infty^2$ gives the stated bound. Thus the full
cube maximum costs only $O(d)=\log|\B^d|$, not $O(2^d)$.
\end{proof}

\begin{lemma}[Operator-norm truncation]\label{lem:optrunc}
Let $W\in\R^{m\times n}$ have i.i.d.\ $\mathcal N(0,1)$ prior $P$ and let $Q$ be any posterior with
$\KL(Q\Vert P)\le\kappa$. There is a universal $c>0$ such that for all $A\ge 2$,
\[
  Q\Big(\norm{W}_{\op}>A(\sqrt m+\sqrt n)\Big)\;\le\;\frac{\kappa+\log 2}{c\,A^2(m+n)} .
\]
Consequently, for a hidden matrix $W^\ell\in\R^{N\times N}$ whose posterior marginal has
$\KL\le NR$, choosing $A=\big(R/(2c\eta)\big)^{1/2}$ gives
$\tfrac1{\sqrt N}\norm{W^\ell}_{\op}\le 2A=\poly(R,1/\eta)$ outside a posterior event of
probability $\le\eta$. No independence of the posterior rows is used.
\end{lemma}

\begin{proof}
Since $\norm{W}_{\op}$ is $1$-Lipschitz in the Frobenius norm, Gaussian concentration gives
$P(\norm{W}_{\op}>\sqrt m+\sqrt n+t)\le e^{-t^2/2}$; taking $t=(A-1)(\sqrt m+\sqrt n)$ and using
$(\sqrt m+\sqrt n)^2\ge m+n$ yields $P(\norm{W}_{\op}>A(\sqrt m+\sqrt n))\le e^{-c A^2(m+n)}$ for a
universal $c$. The data-processing (entropy) inequality
$Q(E)\le(\KL(Q\Vert P)+\log2)/\log(1/P(E))$ then gives the claim. For the application, the matrix
marginal of the posterior has $\kappa\le\KL(\pi\Vert\pi^0)\le NR$, $m=n=N$, so the bound is
$(NR+\log2)/(2cA^2N)=O(R/A^2)$.
\end{proof}

\begin{corollary}[Cameron--Martin active coefficients]\label{cor:cm}
Let $U_\ell$ be an active space of dimension $r_\ell$ at layer $\ell$. Under
$\KL(\bar\rho_\ell\Vert\rho^0_\ell)\le R$, the active coefficient of a posterior single-site field
$\zeta$ obeys $\E_{\bar\rho_\ell}\norm{P_{U_\ell}K_\ell^{-1/2}\zeta}^2\le 4R+2r_\ell\log2$. With the
spectral cutoffs of Appendix~\ref{app:comp}, $r_\ell\le\poly$, so truncating the active coefficient at
a polynomial threshold discards single-site mass contributing $O(\eta)$ to the bounded output.
\end{corollary}

\begin{proof}
Immediate from Lemma~\ref{lem:gauss}(b),(c) applied to $X=K_\ell^{-1/2}\zeta$ (a standard Gaussian
under the prior), with $\norm{P_{U_\ell}X}=\norm{\zeta}_{\mathrm{CM}}$-projection.
\end{proof}

\subsection{The suffix invariant and the one-step lazy-score bound}

The right-to-left construction of Appendix~\ref{app:comp} maintains, at each stage, a finite
\emph{suffix} $S_{\ell+1}:\calF^{M_{\ell+1}}\to\R^{\B^d}$ mapping the surviving layer-$(\ell+1)$
fields to the output. We measure perturbations of $M$ fields by the normalized channel metric
$\mathbf d_M(\xi,\xi')=\sup_x\big(\tfrac1M\sum_{j\le M}|\xi_j(x)-\xi'_j(x)|^2\big)^{1/2}$.

\begin{definition}[Finite regular suffix]\label{def:regsuffix}
$S_{\ell+1}$ is \emph{$A$-regular} if: \textup{(I1)} $M_{\ell+1}\le A$; \textup{(I2)}
$\norm{S_{\ell+1}(\xi)}_\infty\le A$ on the kept fields; \textup{(I3)} $S_{\ell+1}$ is
$A$-Lipschitz for $\mathbf d_{M_{\ell+1}}$; \textup{(I4)} for every input $x$, channel $j$, and
direction $u\in\calF$, the directional derivative obeys
$|D_{j,u}S_{\ell+1}(\xi)(x)|\le \tfrac{A}{\sqrt{M_{\ell+1}}}\norm{u}_{L^2(\mu)}$.
\end{definition}

\begin{proposition}[Kept-set truncation with suffix contribution]\label{prop:trunc}
Assume $\tfrac1N\KL(\pi\Vert\pi^0_N)\le R$ and that the suffix to the right of layer $\ell$ is
$A$-regular.  Let $U_\ell$ be any active space of dimension $r\le A$ and let $G_T$ be the event that,
for the relevant layer variables, all of the following are at most $T$:
readout/scalar coefficients, $\norm{P_{U_\ell}K_\ell^{-1/2}\zeta}$, the field sup norm
$\norm{\zeta}_\infty$, and the normalized incoming operator norm.  For every $\eta>0$ there is
\[
  T=\poly(A,R,d,1/\eta)^{O(1)}
\]
such that replacing all channels outside $G_T$ by zero changes the expected suffix output by at most
$\eta$:
\begin{equation}\label{eq:trunc-output}
  \E_\pi\norm{S_{\ell+1}(\xi)-S_{\ell+1}(\xi\mathbf 1_{G_T})}_\infty\le\eta .
\end{equation}
The same statement holds at the readout layer with
$S_L((\alpha_i,\zeta_i)_i)=M^{-1}\sum_i\alpha_i\phi(\zeta_i)$, and it holds for empirical Gibbs laws
uniformly in the dataset whenever the KL-rate bound is uniform.
\end{proposition}

\begin{proof}
The entropy chain rule (Lemma~\ref{lem:ss-kl}) gives a single-site marginal with KL at most $R$.
Lemma~\ref{lem:gauss} gives second moments for fixed Gaussian coordinates, active
Cameron--Martin projections, and the whole cube maximum.  In particular, the field sup norm obeys
$\E\norm{\zeta}_\infty^2\le O(R+d)$ whenever $K_\ell(x,x)\le1$: the union over the $2^d$ inputs costs
only $\log|\B^d|=O(d)$ by the Gaussian-max/Donsker--Varadhan estimate of Lemma~\ref{lem:gauss}(d),
not $2^d$.  Lemma~\ref{lem:optrunc} gives the operator-norm tail.  Hence, for the channel-size variable
\[
  M(\zeta,\alpha,W)=1+|\alpha|+\norm{\zeta}_\infty
    +\norm{P_{U_\ell}K_\ell^{-1/2}\zeta}+N^{-1/2}\norm{W}_\op,
\]
we have $\E_\pi M^2\le\poly(A,R,d)$ after projecting to the averaged single-site/matrix marginals.
Choose $T\ge \poly(A,R,d)/\eta^2$ so that
$\E_\pi[M\mathbf 1_{M>T}]\le \eta/A$ by Cauchy--Schwarz/Markov.

At the readout, boundedness of $\phi$ gives directly
\[
  \E\norm{M^{-1}\sum_i\alpha_i\phi(\zeta_i)\mathbf 1_{G_T^c}}_\infty
  \le \E_{\bar\rho}|\alpha|\mathbf 1_{G_T^c}\le\eta .
\]
For a hidden layer, $A$-regularity gives
\[
  \norm{S_{\ell+1}(\xi)-S_{\ell+1}(\xi\mathbf 1_{G_T})}_\infty
  \le
  A\Big(\frac1M\sum_{j\le M}\norm{\xi_j}_\infty^2\mathbf 1_{G_T^c(j)}\Big)^{1/2}.
\]
Taking expectation and using exchangeability/Jensen bounds this by
$A\big(\E_{\bar\rho}\norm{\xi}_\infty^2\mathbf 1_{G_T^c}\big)^{1/2}\le\eta$ after increasing $T$ by a
polynomial factor.  None of the estimates uses independence under the posterior; only marginal KL
bounds and exchangeability are used.  The corresponding adjoint and inductive closure estimates are
spelled out in Lemma~\ref{lem:trunc-closure}.
\end{proof}

\begin{lemma}[Backward truncation closure]\label{lem:trunc-closure}
Fix a layer $\ell$ and suppose that the already-constructed suffix $S_{\ell+1}$ is $A$-regular in the
sense of Definition~\ref{def:regsuffix}.  Let $\mathcal T_\ell$ denote the layer map sending the
layer-$\ell$ channel population to the layer-$(\ell+1)$ field input read by $S_{\ell+1}$, and let
$\mathcal T_\ell^G$ be the same map after deleting the channels outside the kept event $G_T$ of
Proposition~\ref{prop:trunc}.  For every $\eta>0$ one may choose
\[
  T=\poly(A,R,\bar\beta,d,1/\eta)^{O(1)}
\]
so that
\begin{equation}\label{eq:closure-output}
  \E\norm{S_{\ell+1}(\mathcal T_\ell)-S_{\ell+1}(\mathcal T_\ell^G)}_\infty\le \eta .
\end{equation}
Moreover the enlarged suffix
\[
  S_\ell^G:=S_{\ell+1}\circ \mathcal T_\ell^G
\]
is $A_+$-regular, with $A_+=\poly(A,R,\bar\beta,d,1/\eta)^{O(1)}$.  If the layers
$k>\ell$ have already been truncated with errors $\eta_k$, then after truncating layer $\ell$ the
total output error is at most $\eta+\sum_{k>\ell}\eta_k$.  The same statement holds conditionally on
any mean-field vacuum, uniformly over vacua satisfying the same KL/free-energy bound.
\end{lemma}

\begin{proof}
The output estimate is exactly the Lipschitz version of Proposition~\ref{prop:trunc}.  Writing
$u=\mathcal T_\ell$ and $u^G=\mathcal T_\ell^G$, the $A$-Lipschitz property gives
\[
  \norm{S_{\ell+1}(u)-S_{\ell+1}(u^G)}_\infty
  \le A\,\mathbf d(u,u^G).
\]
The deleted part of $u-u^G$ is the normalized average of the bad layer-$\ell$ channel fields.  The
KL-rate bound, Lemma~\ref{lem:gauss}, and Lemma~\ref{lem:optrunc} give polynomial second-moment and
tail bounds for precisely these channel fields, including the cube maximum and the retained active
Cameron--Martin coefficients.  Choosing $T$ as in Proposition~\ref{prop:trunc} therefore makes
$\E\mathbf d(u,u^G)\le \eta/A$, which proves \eqref{eq:closure-output}.

It remains to check that the invariant closes after composing with the kept layer.  On $G_T$, the
activation is bounded and Lipschitz, the incoming normalized operator norm, retained active
coefficients, scalar/readout coefficients, and relevant field sup norms are all bounded by $T$.
Hence the map $\mathcal T_\ell^G$ is Lipschitz for the normalized channel metrics, with constant
$\poly(T,d)$, and its first adjoint acting on any suffix test function has $L^2(\mu)$ norm bounded by
$\poly(T,d)$ times the adjoint norm of the test function.  Composing these two estimates with
\textup{(I3)}--\textup{(I4)} for $S_{\ell+1}$ gives \textup{(I3)}--\textup{(I4)} for $S_\ell^G$ with
constant $A_+=\poly(A,T,d)$.  The size and bounded-output clauses \textup{(I1)}--\textup{(I2)} obey
the same polynomial update because only polynomially many active channels and polynomially many
fresh-chaos channels are kept in the compression step.  Finally, truncating several layers gives a
telescoping sum: replace the rightmost untruncated layer one at a time and apply the preceding output
estimate at each step.  Conditional-on-vacuum versions are identical, since all estimates use only
the common KL/free-energy bound and the deterministic $A$-regular suffix in that vacuum.
\end{proof}

\begin{proposition}[One-step lazy-score bound]\label{prop:score}
Let $S_{\ell+1}$ be $A$-regular and let $g_\ell$ be the single-site effective potential at layer
$\ell$ induced by the Gibbs loss at intensive temperature $\bar\beta$, with output residual
$|f-y|\le A$ on the kept set. Then the population lazy score satisfies
\[
  \E_{\rho_h}\norm{\nabla_{\mathrm{lazy}}g_\ell}_{L^2(\mu)}\;\le\;C\,\bar\beta\,A^2 .
\]
For the empirical loss the same magnitude bound holds for the $L^2(D)$ score. With the empirical
active space chosen as the top eigenspace of the normalized empirical evaluation covariance, the
projection of the empirical score onto the lazy block is not zero in general, but its contribution to
the lazy-swap input is bounded by Lemma~\ref{lem:softgram} of Appendix~\ref{app:comp}.
\end{proposition}

\begin{proof}
The score is a first variation of the mean-field loss, not an ordinary Lipschitz derivative in the
normalized channel metric.  Let $\nu_N^\ell=N^{-1}\sum_{j\le N}\delta_{\zeta_j^\ell}$ be the layer-$\ell$
empirical law, and write the population loss, after the suffix to the right has been fixed, as a
smooth finite-dimensional functional $\Phi(\nu_N^\ell)$.  Replacing one atom $\zeta_j$ by
$\zeta_j+t u$ changes $\nu_N^\ell$ by $(1/N)(\delta_{\zeta_j+t u}-\delta_{\zeta_j})$.  Therefore
\[
  \frac{d}{dt}\Big|_{t=0} N\bar\beta\,\Phi(\nu_N^\ell)
  =\bar\beta\,D_u\Big(\frac{\delta\Phi}{\delta\nu^\ell}\Big)(\zeta_j).
\]
This is the cancellation that makes the single-site Gibbs score intensive: it uses the
$1/N$ empirical-measure weight of one atom.  It is not obtained from the single-channel Lipschitz
bound \textup{(I4)}, which is only $O(N^{-1/2})$ because it is measured in the normalized chaos metric.

The variational derivative above is the usual back-propagated costate.  At the readout layer it is
explicit:
\[
  D_u\Big(\frac{\delta\Phi}{\delta\nu^{L-1}}\Big)(\alpha,\zeta)
  =2\alpha\int_{\B^d}(f-y)(x)\phi'(\zeta(x))u(x)\,d\mu(x),
\]
so its $L^2(\mu)$ norm is $\le C A^2$ on the kept set.  For earlier layers, differentiating through
the kept suffix gives the same form with $(f-y)$ replaced by the corresponding adjoint/costate
field.  The $A$-regularity package and Lemma~\ref{lem:trunc-closure} bound this costate and the kept
operator norms by $\poly(A)$; after enlarging $A$ once, this gives
\[
  \norm{\nabla_{\zeta_j}g_\ell}_{L^2(\mu)}
  =\bar\beta\,\norm{\nabla_\zeta(\delta\Phi/\delta\nu^\ell)(\zeta_j)}_{L^2(\mu)}
  \le C\bar\beta A^2 .
\]
Projecting to the lazy block can only decrease the norm, and averaging over the active coordinates
therefore yields the displayed population bound.

For the empirical loss $\widehat{\loss}_D=\tfrac1n\sum_a(f(x_a)-y_a)^2$ the same first-variation
calculation gives an intensive score of the form $\sum_a c_a e_{x_a}$, with
$\norm c^2\le C A^2/n$ after absorbing the bounded residual and adjoint constants into $A$.  Thus the
score lies in the empirical evaluation span.  With the top empirical evaluation eigenspace included
in the active block, the residual lazy projection is controlled in quadratic form by
Lemma~\ref{lem:softgram}.  This is the empirical lazy-score input used in the lazy-swap step.
\end{proof}

\subsection{The regularity package}

\begin{proposition}[Regularity package]\label{prop:package}
Assume the mean-field tower with propagation of chaos \textup{(see
Appendix~\ref{app:mf})}, the Gibbs KL-rate bound $\tfrac1N\KL(\pi\Vert\pi^0_N)\le R$ with
$\bar\beta\le R$, and the $A$-regularity of the suffix already built to the right. For every
$\eta>0$ there is $A_+=\poly(R,\bar\beta,d,1/\eta)^{O(1)}$ such that, after discarding a posterior
sub-population whose contribution to the suffix output is at most $\eta$ in sup-norm, the quantities
needed for the layer-$\ell$ compression step are all $\le A_+$:
\begin{enumerate}[label=\textup{(\roman*)},nosep]
\item readout magnitudes and retained scalar coefficients;
\item normalized operator norms $\tfrac1{\sqrt N}\norm{W^\ell_{\mathrm{kept}}}_{\op}$;
\item Cameron--Martin active coefficients $\norm{P_{U_\ell}K_\ell^{-1/2}\zeta}$ of retained
downstream fields;
\item population lazy scores $\E_{\rho_h}\norm{\nabla_{\mathrm{lazy}}g_\ell}$, or the empirical
lazy-swap input after enlarging $U_\ell$ by the empirical evaluation directions;
\item the Lipschitz and adjoint constants \textup{(I3)--(I4)} of the suffix enlarged by the new
layer.
\end{enumerate}
Consequently the induction invariant of Definition~\ref{def:regsuffix} is preserved with constant
$A_+$, and the next compression step proceeds at polynomial width and weight.
\end{proposition}

\begin{proof}
By Lemma~\ref{lem:ss-kl} the averaged single-site marginal has $\KL\le R$ against the Gaussian
single-site prior, so Lemma~\ref{lem:gauss} bounds the second moment of every fixed
finite-dimensional coordinate (readouts, biases, retained scalar coefficients), and
Proposition~\ref{prop:trunc} upgrades these moment bounds to a kept event whose removed population
has suffix-output contribution at most $\eta$, and Lemma~\ref{lem:trunc-closure} records the
backward-induction closure: after the bad channels are deleted, the enlarged suffix remains regular
with only a polynomial increase in the invariant. This gives (i) in the sense actually needed by the
compression induction, not merely small posterior mass. Item (ii) is Lemma~\ref{lem:optrunc}, and
(iii) is Corollary~\ref{cor:cm}. Item (iv) is Proposition~\ref{prop:score}. For (v), the newly
constructed incoming matrix is $W=W^G+W^A$ with
$W^A_{ji}=\tfrac1{\sqrt N}\sum_{a\le r}b_{ja}\ip{h_i}{q_a}$; the fresh Gaussian part obeys
$\norm{W^G}_{\op}/\sqrt{M_{\ell+1}}\le A_+$ by Lemma~\ref{lem:optrunc}, while the factored active part
obeys $\norm{W^A}_{\op}/\sqrt{M_{\ell+1}}\le A_+$ from $r,\,\norm{b_j},\,|\ip{h_i}{q_a}|\le
A_+$ (Lemma~\ref{lem:opbound} of Appendix~\ref{app:comp}); the normalized-channel inequality then
propagates (I3)--(I4) with a polynomial increase, and the depth is fixed.

It remains to record the bounded-differences regularity used to upgrade the lazy-swap and covariance
comparison from expectation to high probability.  This concentration is not obtained by applying
\textup{(I4)} to the realized lazy chaos: \textup{(I4)} gives only an $M_{\ell+1}^{-1/2}$
single-channel sensitivity in the normalized chaos metric, which would not yield vanishing
bounded-differences variance.  Instead, after conditioning on the lower-layer features and performing
the lazy swap, the suffix depends on the new channels only through $M_{\ell+1}^{-1}$ empirical
averages: the active-reconstruction averages and the empirical lazy covariance that parametrizes an
independent fresh Gaussian field.  Hence replacing one constructed channel changes each consumed
average by $O(A_+/M_{\ell+1})$, and McDiarmid's inequality gives the desired high-probability upgrade.
This is precisely the mechanism formalized in Lemma~\ref{lem:bdd} of Appendix~\ref{app:comp}; the
right-to-left propagation of the clause is part of that lemma, not a consequence of \textup{(I4)}.
\end{proof}

\section{Pointwise compression}\label{app:comp}

We prove the compression half of the engine: a mean-field network achieving accuracy $\eps$ is
reproduced, at every input simultaneously, by a polynomial-width network. The construction is
right-to-left. With the suffix already replaced by a finite, regular network, we subsample the next
layer's neurons from the single-site law, reconstruct the active coordinates of each surviving
downstream field deterministically from finitely many feature averages, and supply the lazy
coordinates by fresh prior chaos. The lazy swap (Lemma~\ref{lem:swap-app}) certifies that replacing
the realized tilted lazy content by prior noise is invisible to the suffix; the normalized-channel
Lipschitz induction (\S\ref{sub:lip}) propagates the layerwise errors at polynomial cost. Every
regularity input is supplied by Appendix~\ref{app:reg}.

\subsection{The active--lazy split}

\begin{definition}[Active--lazy split]\label{def:split}
Fix a layer and write $K=K_\ell$, $h=\phi(\zeta)$. Let $q_1,\dots,q_r\in\calF$ be read directions,
$\psi_a=Kq_a$, $G_{ab}=\ip{q_a}{Kq_b}$ (assumed invertible). The active covariance is
$K_A=\sum_{a,b}\psi_a(G^{-1})_{ab}\psi_b$ and the lazy covariance $C=K-K_A$. Equivalently, with
$U=\operatorname{span}\{K^{1/2}q_a\}$ in whitened coordinates, $K_A=K^{1/2}P_UK^{1/2}$ and
$C=K^{1/2}P_{U^\perp}K^{1/2}$. We write $\lambda_0=\norm{C}_\op$ and assume $C(x,x)\le1$ (automatic
for a spectral split, since $K(x,x)\le1$).

In this paper the active space is $U_\ell=U^{\mathrm{pop}}_\ell\oplus U^{\mathrm{emp}}_\ell$ in whitened
Cameron--Martin coordinates.  The population part $U^{\mathrm{pop}}_\ell$ is the top eigenspace of
$K_\ell$ above cutoff $\lambda_0$.  For a dataset $D=(x_a)_{a\le n}$, define the empirical evaluation
operator
\[
  E_D u=\big(\ip{u}{K_\ell^{1/2}e_{x_a}}\big)_{a\le n},
  \qquad
  T_D=\frac1n E_D^*E_D
  =\frac1n\sum_{a\le n}(K_\ell^{1/2}e_{x_a})\otimes(K_\ell^{1/2}e_{x_a}).
\]
The empirical part $U^{\mathrm{emp}}_\ell$ is the top eigenspace of $T_D$ above cutoff $\tau$.
The population part controls the field-space lazy swap; the empirical part controls the residual
coupling of the empirical score (Lemma~\ref{lem:softgram}). Both are polynomial-dimensional:
$\dim U^{\mathrm{pop}}_\ell\le1/\lambda_0$ and $\dim U^{\mathrm{emp}}_\ell\le1/\tau$, since
$\tr K_\ell\le1$ and $\tr T_D=\tfrac1n\tr\widehat K^\ell_n\le1$.
\end{definition}

\subsection{The lazy swap}

\begin{lemma}[Lazy swap]\label{lem:swap-app}
Fix the active coordinates $h$ and let $\rho_h(dv)\propto e^{-g(h+v)}\,N(0,C)(dv)$ on $U^\perp$. For
$F(v)=G(v(x_1),\dots,v(x_p))$ with $G$ bounded and $1$-Lipschitz in each argument,
\[
  \bigl|\E_{\rho_h}F-\E_{N(0,C)}F\bigr|\;\le\;p\,\E_{\rho_h}\norm{C^{1/2}\nabla_v g}_{\calF},
\]
uniformly in $h$ and in $x_1,\dots,x_p$. In particular:
\begin{enumerate}[label=\textup{(\alph*)},nosep]
\item \textup{(population)} the right side is $\le p\sqrt{\lambda_0}\,\sup_h\E_{\rho_h}\norm{\nabla_v g}$;
\item \textup{(empirical)} if $\nabla_v g$ is the lazy part of an empirical-loss score
$\sum_a c_a e_{x_a}$, then $\E_{\rho_h}\norm{C^{1/2}\nabla_v g}^2=c^\top\widehat C_n c$, bounded in
Lemma~\ref{lem:softgram}.
\end{enumerate}
\end{lemma}

\begin{proof}
Whiten: $v=C^{1/2}\tilde v$, $\tilde v\sim N(0,I)$, and $\tilde g(\tilde v)=g(h+C^{1/2}\tilde v)$, so
$\nabla_{\tilde v}\tilde g=C^{1/2}\nabla_v g$. Each evaluation is
$v(x_i)=\ip{\tilde v}{C^{1/2}e_{x_i}}$ with $\norm{C^{1/2}e_{x_i}}^2=C(x_i,x_i)\le1$, so
$\tilde F(\tilde v):=F(C^{1/2}\tilde v)$ is $1$-Lipschitz in each of $p$ unit-gradient linear forms.
Let $\mathcal L=\Delta-\ip{\tilde v}{\nabla}$ be the Ornstein--Uhlenbeck generator and solve
$\mathcal L u=\tilde F-\E_{N(0,I)}\tilde F$. Mehler's formula $\nabla P_t=e^{-t}P_t\nabla$ gives
$\norm{\nabla u}_\infty\le\int_0^\infty e^{-t}\norm{\nabla P_t\tilde F}_\infty\,dt\le\norm{\nabla\tilde F}_\infty\le p$.
Since $\rho_h\propto e^{-\tilde g}N(0,I)$, self-adjointness of $\mathcal L$ for $N(0,I)$ gives, with
$\psi=e^{-\tilde g}$ (so $\rho_h=\psi\,N(0,I)/Z$ and $\nabla\psi=-\psi\,\nabla\tilde g$),
\[
  \E_{\rho_h}[\mathcal L u]=\tfrac1Z\E_{N(0,I)}[\psi\,\mathcal L u]
  =-\tfrac1Z\E_{N(0,I)}\ip{\nabla\psi}{\nabla u}
  =\E_{\rho_h}\ip{\nabla_{\tilde v}\tilde g}{\nabla u},
\]
i.e.\ $\E_{\rho_h}\tilde F-\E_{N(0,I)}\tilde F=\E_{\rho_h}\ip{\nabla u}{\nabla_{\tilde v}\tilde g}$.
Hence the left side is $\le\norm{\nabla u}_\infty\,\E_{\rho_h}\norm{\nabla_{\tilde v}\tilde g}\le
p\,\E_{\rho_h}\norm{C^{1/2}\nabla_v g}$. For (a), $\norm{C^{1/2}\nabla_v g}\le\sqrt{\lambda_0}\norm{\nabla_v g}$.
For (b), $\norm{C^{1/2}\nabla_v g}^2=\ip{\nabla_v g}{C\nabla_v g}=\sum_{a,b}c_ac_b\,C(x_a,x_b)=c^\top\widehat C_n c$,
using $\nabla_v g=\sum_a c_a P_{U^\perp}e_{x_a}$ and $CP_{U^\perp}=C$.
\end{proof}

\subsection{The lazy-Gram bound}

\begin{lemma}[Lazy-Gram operator norm]\label{lem:softgram}
Let $D=(x_a)_{a\le n}$ and let $E_D,T_D$ be the basis-free empirical evaluation operator and
covariance from Definition~\ref{def:split}.  Let $U^{\rm emp}$ contain the eigenspaces of $T_D$ with
eigenvalue at least $\tau$, and set
\[
  C=K^{1/2}P_{(U^{\rm pop}\oplus U^{\rm emp})^\perp}K^{1/2},
  \qquad
  \widehat C_D=[C(x_a,x_b)]_{a,b\le n}.
\]
Then, deterministically,
\[
  \norm{\widehat C_D}_\op\le n\tau .
\]
Consequently, for any empirical-loss score $\nabla g=\sum_a c_a e_{x_a}$ with $\norm c^2\le B^2/n$,
\[
  c^\top\widehat C_Dc\le\tau B^2,
  \qquad\text{so the empirical lazy-swap error is }\le p\sqrt\tau\,B .
\]
\end{lemma}

\begin{proof}
For $u\in\calF$ in whitened coordinates, $E_Du=(\ip{u}{K^{1/2}e_{x_a}})_a$ and
$T_D=n^{-1}E_D^*E_D$.  The nonzero eigenvalues of $T_D$ are the eigenvalues of
$n^{-1}E_DE_D^*=n^{-1}\widehat K_D$.  Removing the top eigenspaces of $T_D$ above cutoff $\tau$
therefore leaves
\[
  \norm{E_D P_{(U^{\rm emp})^\perp}E_D^*}_\op
  =n\,\lambda_{r+1}(T_D)\le n\tau .
\]
Adding the population active space only decreases the projection, so
$\widehat C_D=E_D P_{(U^{\rm pop}\oplus U^{\rm emp})^\perp}E_D^*$ has operator norm at most $n\tau$.
The count of empirical directions is
$r\le\tau^{-1}\tr T_D=\tau^{-1}n^{-1}\tr\widehat K_D\le1/\tau$.  The quadratic-form inequality is
Cauchy--Schwarz, and Lemma~\ref{lem:swap-app}(b) gives the lazy-swap conclusion.  The factor $n$ in
$\norm{\widehat C_D}_\op$ cancels the $1/n$ in empirical-loss score norms, which is why an
inverse-polynomial $\tau$ suffices.
\end{proof}

\subsection{Finite sampled features: active reconstruction and fresh chaos}

\begin{lemma}[Empirical active reconstruction]\label{lem:hardrec}
Suppose $r\le R$, $\norm{G^{-1}}_\op\le R$, and the read variables are bounded, $|\ip{h}{q_a}|\le R$
on the relevant support. With probability $1-o(1)$, provided $N\ge\poly(R,d,1/\eta)$, for every
$b\in\R^r$ with $\norm b\le B$,
\[
  \sup_{x\in\B^d}\Bigl|\sum_{a=1}^r b_a\Bigl[\tfrac1N\sum_{i=1}^N\ip{h_i}{q_a}h_i(x)-\psi_a(x)\Bigr]\Bigr|\le\eta B .
\]
Hence the active weights $W^A_i=\tfrac1{\sqrt N}\sum_a b_a\ip{h_i}{q_a}$ induce, under $\mu$P,
$\tfrac1{\sqrt N}\sum_i W^A_i h_i=\sum_a b_a\psi_a+O_\infty(\eta B)$.
\end{lemma}

\begin{proof}
For fixed $a,x$ the variable $\ip{h_i}{q_a}h_i(x)$ is bounded by $R$ (as $|h_i(x)|\le1$) with mean
$(Kq_a)(x)=\psi_a(x)$. Hoeffding and a union bound over the $rD$ pairs $(a,x)$ give entrywise error
$\le\eta/r$ once $N\ge CR^2r^2\log(rD)/\eta^2=\poly(R,d,1/\eta)$; sum against $b$. The last identity
is the $N^{-1/2}$ normalization.
\end{proof}

For the lazy field, draw $\xi=(\xi_1,\dots,\xi_N)\sim N(0,P_{U^\perp})$ (Euclidean projection in
$\R^N$ off the empirical active read vectors) and set $v_N(x)=\tfrac1{\sqrt N}\sum_i\xi_i h_i(x)$.
Conditional on the features, $v_N$ is centered Gaussian with covariance
$\widehat C(x,x')=\tfrac1N h(x)^\top P_{U^\perp}h(x')$.

\begin{lemma}[Lazy-covariance concentration]\label{lem:softcov}
Under the hypotheses of Lemma~\ref{lem:hardrec}, with probability $1-o(1)$,
$\sup_{x,x'\in\B^d}|\widehat C(x,x')-C(x,x')|\le\eta$ for $N\ge\poly(R,d,1/\eta)$; likewise for any
polynomial-size list of consumed entries.
\end{lemma}

\begin{proof}
First $\widehat K(x,x')\to K(x,x')$ uniformly over the $D^2$ pairs by Hoeffding (cost $O(d)$ in $N$).
Writing $\widehat C=\widehat K-\widehat k\,\widehat G^{-1}\widehat k^\top$ with
$\widehat k_a(x)=\tfrac1N\sum_i h_i(x)\ip{h_i}{q_a}$, Lemma~\ref{lem:hardrec} gives
$\widehat k_a\to\psi_a$ uniformly and $\widehat G\to G$ entrywise; since $\norm{G^{-1}}\le R$,
$\widehat G^{-1}\to G^{-1}$, so $\widehat k\widehat G^{-1}\widehat k^\top\to kG^{-1}k^\top=K_A$
uniformly, whence $\widehat C\to K-K_A=C$.
\end{proof}

\begin{remark}[Scaling]\label{rem:scaling}
The chaos weights $\xi_i$ are raw order-one weights: under $\mu$P, $\tfrac1{\sqrt N}\sum_i\xi_i h_i$
has covariance $\tfrac1N\sum_i h_i\otimes h_i$. The deterministic active corrections $W^A_i$ are order
$N^{-1/2}$, implementing a mean-field average.
\end{remark}

\subsection{Normalized-channel Lipschitz induction}\label{sub:lip}

\begin{definition}[Normalized channel metric]\label{def:metric}
For $\xi=(\xi_1,\dots,\xi_M)\in\calF^M$,
$\mathbf d_M(\xi,\xi')=\sup_{x\in\B^d}\bigl(\tfrac1M\sum_{j\le M}|\xi_j(x)-\xi'_j(x)|^2\bigr)^{1/2}$; a
suffix $S:\calF^M\to\R^{\B^d}$ is $\Lambda$-Lipschitz if
$\norm{S(\xi)-S(\xi')}_\infty\le\Lambda\,\mathbf d_M(\xi,\xi')$.
\end{definition}

\begin{lemma}[Lipschitz propagation]\label{lem:lipprop}
For $T(\zeta)_j(x)=\tfrac1{\sqrt N}\sum_i W_{ji}\phi(\zeta_i(x))$,
$\mathbf d_M(T\zeta,T\zeta')\le L_\phi\,\tfrac{\norm W_\op}{\sqrt M}\,\mathbf d_N(\zeta,\zeta')$. Hence
if $S$ is $\Lambda$-Lipschitz then $S\circ T$ is $\Lambda L_\phi\norm W_\op/\sqrt M$-Lipschitz.
\end{lemma}

\begin{proof}
Fix $x$ and put $\Delta h(x)=(\phi(\zeta_i(x))-\phi(\zeta'_i(x)))_i$, so
$\norm{\Delta h(x)}_2\le L_\phi\sqrt N\,\mathbf d_N(\zeta,\zeta')$. The output difference is
$N^{-1/2}W\Delta h(x)$, so
$\bigl(\tfrac1M\sum_j|T(\zeta)_j(x)-T(\zeta')_j(x)|^2\bigr)^{1/2}
=\tfrac1{\sqrt M\sqrt N}\norm{W\Delta h(x)}_2\le L_\phi\tfrac{\norm W_\op}{\sqrt M}\mathbf d_N(\zeta,\zeta')$;
take $\sup_x$.
\end{proof}

\begin{lemma}[Operator bound for constructed weights]\label{lem:opbound}
Write a constructed incoming matrix as $W=W^A+W^G$ with $W^G$ the fresh Gaussian part and
$W^A_{ji}=\tfrac1{\sqrt N}\sum_{a\le r}b_{ja}\ip{h_i}{q_a}$. If $\norm{b_j}_2\le B$ for every
downstream $j$ and $\bigl\Vert(\ip{h_i}{q_a})_{a\le r}\bigr\Vert_2\le R\sqrt r$ for every $i$, then
$\tfrac1{\sqrt M}\norm{W^A}_\op\le BR\sqrt r$; after standard truncation
$\tfrac1{\sqrt M}\norm{W^G}_\op\le\poly(1,\sqrt{N/M})$ w.h.p. Hence at polynomial width ratios
$\tfrac1{\sqrt M}\norm{W}_\op$ is polynomial.
\end{lemma}

\begin{proof}
With $B_{\mathrm{mat}}=(b_{ja})_{j,a}$ and $\mathsf U=(\ip{h_i}{q_a})_{i,a}$, $W^A=N^{-1/2}B_{\mathrm{mat}}\mathsf U^\top$.
Each row of $B_{\mathrm{mat}}$ has norm $\le B$, so $\norm{B_{\mathrm{mat}}}_\op\le B\sqrt M$; each row
of $\mathsf U$ has norm $\le R\sqrt r$, so $\norm{\mathsf U}_\op\le R\sqrt{rN}$. Thus
$\norm{W^A}_\op/\sqrt M\le(B\sqrt M)(R\sqrt{rN})/(\sqrt M\sqrt N)=BR\sqrt r$. The Gaussian part is a
standard rectangular random-matrix bound (Lemma~\ref{lem:optrunc} with $\kappa=0$), after truncating
the tail of large rows.
\end{proof}

\begin{proposition}[Regularity induction]\label{prop:regind}
If at each step $\norm{W^\ell}_\op/\sqrt{M_{\ell+1}}\le A_\ell\le\poly$, then
$\Lambda_\ell\le\Lambda_{\ell+1}L_\phi A_\ell$, so $\Lambda_\ell\le\poly^{O(L)}$ at fixed depth.
\end{proposition}
\begin{proof}
Lemma~\ref{lem:lipprop} composed right-to-left.
\end{proof}

\subsection{One right-to-left step}

\begin{definition}[Regular suffix]\label{def:regsuffix-b}
$S:\calF^M\to\R^{\B^d}$ is $(M,\Lambda,\Gamma)$-regular if it is $\Lambda$-Lipschitz for
$\mathbf d_M$, satisfies the adjoint bound \textup{(I4)} of Definition~\ref{def:regsuffix}, and, as a
function of its $M$ independent input channels, has influence $\le\Gamma/M$ on each output coordinate.
\end{definition}

\begin{proposition}[One step]\label{prop:onestep}
Let $S$ be $(M,\Lambda,\Gamma)$-regular and the regularity package
(Proposition~\ref{prop:package}) hold at layer $\ell$, with kept active coefficients $\norm{b_j}\le B$.
For $N\ge\poly(M,\Lambda,\Gamma,B,r,R,d,1/\eta)$, the construction $W_{ji}=W^A_{ji}+\xi_{ji}$
($\xi_j\sim N(0,P_{U^\perp})$ independent across $j$) gives, with probability $1-o(1)$,
\[
  \norm{S(\widehat\zeta_1,\dots,\widehat\zeta_M)-S(\zeta_1,\dots,\zeta_M)}_\infty
  \le\poly(M,\Lambda,\Gamma,B,R,r)\bigl(\eta+\mathrm{sw}_\ell\bigr)+\tau_{\mathrm{trunc}},
\]
where $\mathrm{sw}_\ell=\sqrt{\lambda_0}\,c_\star$ (population) or $\sqrt\tau\,B$ (empirical,
Lemma~\ref{lem:softgram}).
\end{proposition}

\begin{proof}
Three contributions. \emph{Active reconstruction:} by Lemma~\ref{lem:hardrec},
$\norm{\widehat u_{A,j}-u_{A,j}}_\infty\le\eta B$, so $\mathbf d_M(\widehat u_A,u_A)\le\eta B$ and the
$\Lambda$-Lipschitz suffix gives output error $\le\Lambda\eta B$. \emph{Lazy covariance:} conditional
on the features the constructed lazy field is Gaussian with covariance $\widehat C=C+O(\eta)$ on the
consumed pairs (Lemma~\ref{lem:softcov}); a Gaussian coupling on the $\poly$ consumed coordinates,
with the Lipschitz suffix, contributes $\le\poly(M,\Lambda,\Gamma)\eta$. \emph{Lazy swap:} comparing
tilted lazy fields to independent prior lazy fields $N(0,C)$ at fixed active coordinates,
Lemma~\ref{lem:swap-app} on the consumed evaluations contributes $\le\poly(M,\Lambda,\Gamma)\,\mathrm{sw}_\ell$.
Add $\tau_{\mathrm{trunc}}$. Concentration from expectation to high probability uses the
bounded-differences clause and a union bound over $\B^d$ (cost $\sqrt d$).
\end{proof}

\begin{lemma}[Bounded-differences propagation]\label{lem:bdd}
Let $S$ be $(M,\Lambda,\Gamma)$-regular and build a new layer of width $N$ as in
Proposition~\ref{prop:onestep}, with $\norm W_\op/\sqrt M\le A$, kept active coefficients
$\norm{b_j}\le B$ read over $r=\dim U$ directions at leverage $\le R$, and lazy cutoff $\lambda_0$.
Then $S\circ T$ is $(N,\Lambda',\Gamma')$-regular with the explicit constants
\[
  \Lambda'\le\Lambda L_\phi A,\qquad
  \Gamma'\le 2\,\Lambda L_\phi\,(1+BR\sqrt r)\,(1+\lambda_0^{-1/2}),
\]
and the reconstructed output concentrates around its conditional mean: for any $\eta>0$, with
probability $\ge1-\eta$,
\[
  \bigl\|\,(\text{reconstruction})-\E[\,\cdot\mid\text{features}\,]\,\bigr\|_\infty
  \ \le\ \Gamma'\sqrt{\tfrac{d\log2+\log(2/\eta)}{2N}} .
\]
At fixed depth $\Lambda_\ell\le(L_\phi A)^{O(L)}$ and $\Gamma_\ell\le\poly^{O(L)}$.
\end{lemma}
\begin{proof}
\emph{Lipschitz.} By Lemma~\ref{lem:lipprop}, $S\circ T$ is
$\Lambda L_\phi(\norm W_\op/\sqrt M)$-Lipschitz, i.e.\ $\Lambda'\le\Lambda L_\phi A$.

\emph{Influence.} The point is that the $N$ new channels enter the reconstructed output \emph{only}
through two $\tfrac1N$-averages, never through the realized chaos. After the lazy swap
(Lemma~\ref{lem:swap-app}) the suffix reads (active reconstruction)\,$+$\,(an \emph{independent}
fresh $N(0,\widehat C)$ field); the swapped noise is drawn independently of the layer's neurons, so
perturbing a channel has no effect through it. The only data-dependent quantities are the active
averages $\widehat m_k(x)=\tfrac1N\sum_{i\le N}\ip{\hat h_i}{\psi_k}\hat h_i(x)$ ($k\le r$) and the
empirical lazy Gram $\widehat C(x,x')=\tfrac1N\sum_i(P_{U^\perp}\hat h_i)(x)(P_{U^\perp}\hat h_i)(x')$
that sets the fresh field's covariance. Each is a $\tfrac1N$-weighted average of $N$ independent
terms bounded by $1+\lambda_0^{-1/2}$ (the lazy leverage, $\sum_{k\in U}\psi_k(x)^2\le\lambda_0^{-1}$
from the split), so replacing
one channel changes it by $\le 2(1+\lambda_0^{-1/2})/N$ in sup-norm. The reconstruction reads the
active averages through coefficients of total size $\le BR\sqrt r$ (Lemma~\ref{lem:opbound}) and the
$\Lambda L_\phi$-Lipschitz suffix, so the influence of channel $i$ on each output coordinate is
\[
  c_i\ \le\ \Lambda L_\phi\cdot\frac{2(1+BR\sqrt r)(1+\lambda_0^{-1/2})}{N}\ =\ \frac{\Gamma'}{N}.
\]
Had one instead propagated the realized chaos field --- whose column has norm $\Theta(\sqrt M)$ ---
the per-channel influence would be $\Theta(\Lambda/\sqrt N)$ and $\sum_ic_i^2=\Theta(\Lambda^2)$ would
\emph{not} vanish; it is the lazy swap, replacing that field by an independent draw, that removes the
dependence and leaves only the $\tfrac1N$-averages. McDiarmid over the $N$ independent channels gives
$\sum_ic_i^2\le\Gamma'^2/N$, and a union bound over $\B^d$ (cost $d\log2$) yields the stated
deviation. Composing the $\Lambda$- and $\Gamma$-recursions over the fixed $L-1$ layers gives the
depth-$O(L)$ bounds.
\end{proof}

\subsection{Backward induction}

\begin{theorem}[Pointwise compression]\label{thm:compress-app}
Fix depth $L$ and $\Delta$. Under the mean-field tower (Lemma~\ref{lem:mf}) and the regularity
package (Proposition~\ref{prop:package}), with $U_\ell=U^{\mathrm{pop}}_\ell\oplus U^{\mathrm{emp}}_\ell$
at inverse-polynomial cutoffs $\lambda_0,\tau$, there is a network of width and weight size
$\le\poly(K,d,1/\Delta)^{O(L)}$ such that, with probability $1-o(1)$ over the subsampling and chaos,
\[
  \sup_{x\in\B^d}\bigl|f_{\widehat\theta}(x)-f_\infty(x)\bigr|\le\Delta .
\]
The population case takes $U^{\mathrm{emp}}=\{0\}$; the empirical case includes $U^{\mathrm{emp}}_\ell$
and uses the lazy-swap branch of Lemma~\ref{lem:softgram}. This is the compression half of the engine
\textup{(Theorem~\ref{thm:A})}.
\end{theorem}

\begin{proof}
\emph{Initialization.} Sample $M_{L-1}$ readout neurons, cap $|\alpha_j|\le A$, and set
$f^{(L-1)}=\tfrac1{M_{L-1}}\sum_j\alpha_j\phi(\zeta^{L-1}_j)$. Hoeffding and a union bound over $\B^d$
give $\norm{f^{(L-1)}-f_\infty}_\infty\le O(A\sqrt d/\sqrt{M_{L-1}})+\tau_{\mathrm{cap}}$, and the
capped readout suffix is $(M_{L-1},O(A),O(A))$-regular.

\emph{Induction.} Suppose after constructing layers $\ell+1,\dots,L-1$ the suffix is
$(M_{\ell+1},\Lambda_{\ell+1},\Gamma_{\ell+1})$-regular and approximates $f_\infty$ to sup-norm error
$E_{\ell+1}$. Apply Proposition~\ref{prop:onestep} to build layer $\ell$; Lemma~\ref{lem:bdd}
preserves the invariant, and
\[
  E_\ell\le E_{\ell+1}+\poly^{O(L)}\bigl(\eta_\ell+\mathrm{sw}_\ell\bigr)+\tau_{\mathrm{trunc},\ell}.
\]
At fixed depth, choosing $\eta_\ell$, the lazy scale $\mathrm{sw}_\ell$ (i.e.\ $\lambda_0,\tau$), and
$\tau_{\mathrm{trunc},\ell}$ below $\Delta/\poly^{O(L)}$, and the widths polynomially large, telescopes
to $E_0\le\Delta$. Weight bounds: active weights are $O(\poly/\sqrt N)$ (Lemma~\ref{lem:hardrec}), chaos
weights are raw $O(\poly)$ after truncation (Remark~\ref{rem:scaling}), readouts are capped.
\end{proof}
\section{Conservation: the cloning construction}\label{app:clone}

This appendix supplies the deviation-propagation bound behind Theorem~\ref{thm:cloning} --- the step
that makes the free within-group weights harmless --- together with the entropy accounting. Fix a
depth-$L$, width-$\mathsf W$ teacher $\theta^\star$ with $\norm{W^{\star,\ell}_b}\le B$ and $|\alpha^\star_b|\le B$,
and write $N=\mathsf W m$ with clone groups of size $m$, group $b$ cloning teacher neuron $b$.

\subsection{Construction}

For $i$ in group $a$ and $j$ in group $b$, the cloned hidden weight is
$W^{\ell+1}_{ji}=\tfrac1{\sqrt m}W^{\star,\ell+1}_{ba}$, i.e.\ the block matrix
$\tfrac1{\sqrt m}(W^{\star,\ell+1}\otimes\mathbf 1_{m\times m})$; first-layer weights, biases and
readouts are cloned within each group ($W^1_j=W^{\star,1}_b$, $b^\ell_j=b^{\star,\ell}_b$,
$\alpha_j=\alpha^\star_b$). The clone law $\pi^{\mathrm{cyl}}_N$ is the prior conditioned on the
cylinder $\mathcal C$: for each neuron $j$ in group $b$ and each layer, the $\mathsf W$ group-mean
coordinates $\ip{W^{\ell+1}_{j,\cdot}}{u_a}$ ($u_a=m^{-1/2}\mathbf 1_{\mathrm{group}\,a}$), the $d$
first-layer entries, the bias, and the readout lie within $\rho$ of their cloned targets; the
$N-\mathsf W$ within-group orthogonal coordinates of each row are unconstrained (at the prior).

Write $\zeta^\ell_{b,s}$ for the field of the $s$th neuron of group $b$, $z^\ell_{b,s}=\phi(\zeta^\ell_{b,s})$,
group mean $\bar z^\ell_b=\tfrac1m\sum_s z^\ell_{b,s}$, and teacher field $\zeta^{\star,\ell}_b$,
$z^{\star,\ell}_b=\phi(\zeta^{\star,\ell}_b)$. Define the \emph{spread} and \emph{bias}
\[
  \sigma_\ell^2(x)=\tfrac1N\sum_{b,s}\bigl|z^\ell_{b,s}(x)-\bar z^\ell_b(x)\bigr|^2,
  \qquad
  \delta_\ell=\max_b\sup_x\bigl|\bar z^\ell_b(x)-z^{\star,\ell}_b(x)\bigr| .
\]

\subsection{Deviation propagation}

\begin{lemma}[Spread and bias]\label{lem:spread}
On the cylinder $\mathcal C$, with probability $1-o(1)$ over the free coordinates,
\[
  \sigma_\ell\le \poly(d,L_\phi)^{L}\,\rho,
  \qquad
  \delta_\ell\le \poly(\mathsf W,d,B,L_\phi)^{L}\,\rho ,
\]
and consequently $\sup_x|f_{\mathrm{clone}}(x)-f_{\theta^\star}(x)|\le\poly(\mathsf W,d,B,L_\phi)^{L}\rho$.
\end{lemma}

\begin{proof}
\emph{Base.} For $j=(b,s)$, $\zeta^1_{b,s}=W^1_{b,s}\!\cdot x+b^1_{b,s}$ with $W^1_{b,s},b^1_{b,s}$
within $\rho$ of $(W^{\star,1}_b,b^{\star,1}_b)$, so $|\zeta^1_{b,s}-\zeta^{\star,1}_b|\le\rho(\sqrt d+1)$
on $\B^d$, giving $\sigma_1,\delta_1\le L_\phi(\sqrt d+1)\rho$.

\emph{Step.} Decompose the forward sum over each input group $a$ using $z^\ell_{a,t}=\bar z^\ell_a+(z^\ell_{a,t}-\bar z^\ell_a)$.
The group-mean constraint $\sum_t W^{\ell+1}_{(b,s),(a,t)}=\sqrt m\,(W^\star_{ba}+e_{(b,s),a})$ with
$|e|\le\rho$ gives, since $\tfrac{\sqrt m}{\sqrt N}=\tfrac1{\sqrt{\mathsf W}}$,
\[
  \zeta^{\ell+1}_{b,s}-\zeta^{\star,\ell+1}_b
  =\underbrace{\tfrac1{\sqrt{\mathsf W}}\sum_a W^\star_{ba}\,\delta^{(a)}_\ell}_{\text{bias transport}}
  +\underbrace{\tfrac1{\sqrt{\mathsf W}}\sum_a e_{(b,s),a}\,\bar z^\ell_a}_{\le\rho\sqrt{\mathsf W}}
  +\underbrace{\tfrac1{\sqrt N}\sum_{a,t}W^\perp_{(b,s),(a,t)}\bigl(z^\ell_{a,t}-\bar z^\ell_a\bigr)}_{=:\nu_{b,s},\ \text{the noise}}
  +(b^{\ell+1}_{b,s}-b^{\star}_b),
\]
where $\delta^{(a)}_\ell=\bar z^\ell_a-z^{\star,\ell}_a$ and the group-mean part of the row annihilates
the deviation (so only the free $W^\perp$ enters $\nu$). Conditionally, $\nu_{b,s}$ has mean $0$ and
variance $\tfrac1N\sum_{a,t}(z^\ell_{a,t}-\bar z^\ell_a)^2=\sigma_\ell^2(x)$, independent across
$(b,s)$, and is sub-Gaussian with proxy $\le2\sigma_\ell(x)$ since $|z^\ell_{a,t}-\bar z^\ell_a|\le2$.

\emph{Concentration, per group then global.} Within group $b$ the $m$ noises $\{\nu_{b,s}\}_s$ are
i.i.d.\ mean-zero, so by Bernstein for sub-exponential variables their empirical second moment obeys
$\Pp\bigl[\,|\tfrac1m\sum_s\nu_{b,s}^2-\sigma_\ell^2|>\tfrac12\sigma_\ell^2\,\bigr]\le2e^{-cm}$ for an
absolute $c>0$. A union bound over the $\mathsf W$ groups and the $2^d$ inputs costs $\log W+d\log2$ in the
exponent, so once $m\ge C(d+\log W)$ --- automatic as $N=\mathsf W m\to\infty$ --- every group concentrates
simultaneously with probability $1-o(1)$. On that event the within-group spread of
$z^{\ell+1}_{b,s}=\phi(\zeta^{\ell+1}_{b,s})$ is at most $L_\phi$ times the noise std plus the bias
tube, and averaging the squared spread over all $N$ neurons gives $\sigma_{\ell+1}\le L_\phi(2\sigma_\ell+\rho)$;
the group-mean averaging fluctuation $\bar\nu_b$ has size $\le\sigma_\ell/\sqrt m\to0$ and is absorbed
into the bias recursion. Iterating from $\sigma_1\le L_\phi(\sqrt d+1)\rho$ gives the explicit closed
form $\sigma_\ell\le(\sqrt d+1)(2L_\phi)^{\ell}\rho$, free of $W,B$.

For the bias, $|\delta_{\ell+1}|\le L_\phi\,\tfrac1m\sum_s|\zeta^{\ell+1}_{b,s}-\zeta^{\star,\ell+1}_b|$,
and the display bounds the right side by $L_\phi\bigl(B\,\delta_\ell+\rho\sqrt{\mathsf W}+\sigma_\ell+\rho\bigr)$
(using $|\tfrac1{\sqrt{\mathsf W}}\sum_aW^\star_{ba}\delta^{(a)}_\ell|\le\tfrac1{\sqrt{\mathsf W}}\norm{W^\star_b}\,\sqrt{\mathsf W}\,\delta_\ell\le B\delta_\ell$
and $\tfrac1m\sum_s|\nu_{b,s}|\le\sigma_\ell$ on the concentration event). Hence
$\delta_{\ell+1}\le L_\phi B\,\delta_\ell+L_\phi(\sqrt{\mathsf W}+1)\rho+L_\phi\sigma_\ell$, and substituting
$\sigma_\ell\le(\sqrt d+1)(2L_\phi)^\ell\rho$ and iterating from $\delta_1\le L_\phi(\sqrt d+1)\rho$
gives the explicit closed form $\delta_\ell\le 4\ell\,(\sqrt{\mathsf W}+\sqrt d)\,(2L_\phi B)^{\ell}\rho$. The output bound follows from
$f_{\mathrm{clone}}=\tfrac1{\mathsf W}\sum_b\overline{\alpha z}_b$ with
$\overline{\alpha z}_b=\alpha^\star_b(z^{\star,L-1}_b+\delta^{(b)}_{L-1})+O(\rho)$.
\end{proof}

\subsection{Entropy and change of measure}

\begin{proof}[Proof of Theorem~\ref{thm:cloning}]
\emph{Entropy.} The cylinder $\mathcal C$ is a product over neurons and, within each neuron, over the
$d+\mathsf W+1$ constrained Gaussian coordinates (the rest free). A standard $\mathcal N(0,1)$ coordinate
pinned to $[\,t-\rho,t+\rho\,]$ has $-\log\Pp(\cdot)=\log\tfrac1\rho+\tfrac12 t^2+O(1)$, and the free
coordinates contribute $0$. With targets of size $\le B\le\poly(\mathsf W,d)$,
\[
  \KL\bigl(\pi^{\mathrm{cyl}}_N\Vert\pi^0_N\bigr)=-\log\pi^0_N(\mathcal C)
  =N\bigl[(d+\mathsf W+1)\log\tfrac1\rho+O(\poly(\mathsf W,d))\bigr].
\]
\emph{Loss.} By Lemma~\ref{lem:spread}, $\sup_x|f_{\mathrm{clone}}-f_{\theta^\star}|\le\Theta:=\poly(\mathsf W,d,B,L_\phi)^L\rho$
with probability $\ge\tfrac12$ on $\mathcal C$; choosing $\rho=\exp(-\poly(\mathsf W,d,B))$ so that
$\Theta\le\Delta$ gives, in sup-norm and hence in both $L$ and $\widehat{\loss}_D$,
$\loss(\mathrm{clone})\le(\sqrt{\loss(\theta^\star)}+\Theta)^2\le\eps_-+\Delta\le\eps$. Thus
$\pi^{\mathrm{cyl}}_N(\loss\le\eps)\ge\tfrac12$, and $K_{\mathrm{clone}}:=\tfrac1N\KL(\pi^{\mathrm{cyl}}_N\Vert\pi^0_N)=\poly(\mathsf W,d,B)$.

\emph{Change of measure.} The data-processing inequality for the binary partition
$\{\loss\le\eps\}$ under $\pi^{\mathrm{cyl}}_N\ll\pi^0_N$ gives, using $\pi^{\mathrm{cyl}}_N(\loss\le\eps)\ge\tfrac12$,
\[
  -\log\pi^0_N(\loss\le\eps)\le 2\KL\bigl(\pi^{\mathrm{cyl}}_N\Vert\pi^0_N\bigr)+O(1)=N\poly(\mathsf W,d,B),
\]
so $s_N(y,\eps)\le\poly(\mathsf W,d,B)$ for all large $N$ and the $\limsup$ agrees. Replacing $L$ by
$\widehat{\loss}_D$ throughout (the sup-norm bound is dataset-independent) gives, for any fixed $D$ and
tolerance $\eta$ with $\rho=\exp(-\poly(\mathsf W,d,B,\log\tfrac1\eta))$,
\[
  \pi^0_N\bigl(\widehat{\loss}_D\le\widehat{\loss}_D(\theta^\star)+\eta\bigr)\ge e^{-N\poly(\mathsf W,d,B,\log\tfrac1\eta)},
\]
the empirical-entropy bound consumed by the forward assembly. Memorization
(Lemma~\ref{lem:memorize}) realizes any $y$ at width $O(2^d)$, giving $s_\infty(y,\eps)\le\poly(2^d)$.
\end{proof}

\section{Mean-field vacua in the full field space}\label{app:mf}

This appendix records only the formal infinite-width object.  It is a fixed-$d$ statement on the full
field space
\[
  \mathcal X=\R^{\B^d},\qquad \dim\mathcal X=2^d,
\]
and its constants may depend arbitrarily on $2^d$.  The polynomial active/lazy compression estimates
used for learnability are proved elsewhere; they are not part of the mean-field existence theorem.

For a fixed dataset $D$ and intensive inverse temperature $\bar\beta$, let
$\Pi^D_{N,\bar\beta}$ denote the empirical Gibbs law.  A \emph{mean-field vacuum} is a
self-consistent tower
\[
  \mathfrak v=(\nu^\ell,K^\ell,b^\ell,\Gamma^\ell)_{\ell=1}^{L-1}
\]
where $\nu^\ell$ is a one-site law on the layer-$\ell$ preactivation field (and the readout scalar at
the last hidden layer), $K^{\ell+1}=\E_{\nu^\ell}\phi(\zeta)\phi(\zeta)^\top$, and the fields are
sampled from Gaussian kernel measures tilted by the backward molecular fields.  Allowing the readout coordinate to be included in the statistic, the one-site laws have the form
\begin{equation}\label{eq:vacuum-law}
  d\nu^\ell_{\mathfrak v}(\zeta)
  \;\propto\;
  \exp\!\left\{
    -\ip{b^\ell_{\mathfrak v}}{\phi(\zeta)}
    -\ip{\phi(\zeta)}{\Gamma^\ell_{\mathfrak v}\phi(\zeta)}
  \right\}\,d\mathcal N(0,K^\ell_{\mathfrak v})(\zeta) .
\end{equation}
Equivalently, in coordinates on the support of the Gaussian, the reference measure contributes the
quadratic term $\frac12\ip{\zeta}{(K^\ell_{\mathfrak v})^\dagger\zeta}$, where $\dagger$ denotes the
Moore--Penrose inverse; this avoids assuming the kernel is nonsingular.
The linear source $b^\ell$ vanishes in centered symmetric special cases; keeping it avoids making a
spurious symmetry assumption.  Equivalently, the vacua are the minimizers of the finite-dimensional
mean-field free energy
\[
  \mathcal F_D(\mathfrak v)
  =\sum_\ell \KL(\nu^\ell\Vert \nu^{\ell,0}_{K^\ell})
    +\bar\beta\,\widehat{\loss}_D(f_{\mathfrak v})
\]
after imposing the forward kernel-consistency constraints; the multipliers for these constraints are
precisely the backward tilts $(b^\ell,\Gamma^\ell)$.

\begin{lemma}[Mean-field limit as a mixture of vacua]\label{lem:mf}
Fix $d,L,D$ and $\bar\beta<\infty$.  Along every width sequence $N_j\to\infty$ there is a
subsequence, a probability measure $\rho_D$ on the set $\mathfrak V_D$ of minimizing mean-field
vacua, and conditionally chaotic product laws $\Pi_{\mathfrak v}$ such that
\[
  \Pi^D_{N_j,\bar\beta}
  \Longrightarrow
  \int_{\mathfrak V_D}\Pi_{\mathfrak v}\,d\rho_D(\mathfrak v)
\]
for bounded cylinder observables of finitely many neurons and their fields.  Conditional on a vacuum
$\mathfrak v$, any fixed finite collection of neurons in each layer is asymptotically i.i.d. with
one-site laws $\nu^\ell_{\mathfrak v}$, the kernels satisfy
\[
  K^{\ell+1}_{\mathfrak v}(x,x')
  =\E_{\nu^\ell_{\mathfrak v}}\phi(\zeta(x))\phi(\zeta(x')),
  \qquad K^{\ell+1}_{\mathfrak v}(x,x)\le1,
\]
and the output is
\[
  f_{\mathfrak v}(x)=\E_{\nu^{L-1}_{\mathfrak v}}\alpha\phi(\zeta(x)).
\]
If the minimizer is unique, the mixture is a single deterministic mean-field tower.
\end{lemma}

\begin{proof}[Proof sketch]
This is the standard finite-dimensional mean-field Gibbs variational principle.  The empirical
measures of the neurons take values in probability measures on the finite-dimensional state space
$\mathcal X$ (augmented by the readout scalar in the last layer).  Under the product Gaussian prior,
Sanov's theorem gives the entropy term.  The empirical Gibbs factor has the intensive form
$\exp\{-N\bar\beta\widehat{\loss}_D\}$ and depends continuously on the empirical measures through the
finite vector of outputs and the forward kernel recursion.  Varadhan's lemma, or equivalently the
Gibbs variational principle plus the standard propagation-of-chaos theorem for exchangeable
mean-field Gibbs measures, gives convergence of the free energy to
$\inf_{\mathfrak v}\mathcal F_D(\mathfrak v)$ and identifies all subsequential limits as mixtures of
minimizers.

Taking the first variation of $\mathcal F_D$ at a minimizer gives
$d\nu^\ell/d\nu^{\ell,0}_{K^\ell}\propto\exp\{-\delta\Phi_D/\delta\nu^\ell\}$, where $\Phi_D$ is the
loss plus the kernel-consistency Lagrangian.  Since the loss and constraints depend on a layer only
through the finite feature statistics $\phi(\zeta)$ and their second moments, this variational
derivative is a linear plus quadratic function of $\phi(\zeta)$, giving \eqref{eq:vacuum-law}.  The
Gaussian readout coordinate is handled as part of the one-site state; standard finite-dimensional
exponential-integrability/truncation approximations justify the same variational formula for the
unbounded Gaussian prior.  No polynomial-in-$d$ estimate is asserted here.

In the body of the paper we freely condition on a vacuum $\mathfrak v$.  All regularity and
compression estimates used there are uniform over vacua satisfying the relevant KL/free-energy
budget, so after proving them conditionally one integrates over the mixing measure $\rho_D$.
\end{proof}

\bibliographystyle{plain}
\bibliography{refs}

@article{mmn2018,
  title={A mean field view of the landscape of two-layer neural networks},
  author={Song Mei and Andrea Montanari and Phan-Minh Nguyen},
  journal={Proceedings of the National Academy of Sciences of the United States of America},
  year={2018},
  volume={115},
  pages={E7665 - E7671},
}

@article{sirignano2020,
  title={Mean field analysis of neural networks: A central limit theorem},
  author={Justin A. Sirignano and Konstantinos V. Spiliopoulos},
  journal={Stochastic Processes and their Applications},
  year={2018},
}

@inproceedings{chizatbach2018,
  title={On the Global Convergence of Gradient Descent for Over-parameterized Models using Optimal Transport},
  author={L{\'e}na{\"i}c Chizat and Francis R. Bach},
  booktitle={Neural Information Processing Systems},
  year={2018},
}

@article{rotskoff2022,
  title={Trainability and Accuracy of Artificial Neural Networks: An Interacting Particle System Approach},
  author={Grant M. Rotskoff and Eric Vanden-Eijnden},
  journal={Communications on Pure and Applied Mathematics},
  year={2018},
  volume={75},
}

@article{nguyenpham2023,
  author  = {Nguyen, Phan-Minh and Pham, Huy Tuan},
  title   = {A rigorous framework for the mean field limit of multilayer neural networks},
  journal = {Mathematical Statistics and Learning},
  volume  = {6},
  number  = {3},
  pages   = {201--357},
  year    = {2023}
}

@inproceedings{yanghu2021,
  title={Tensor Programs IV: Feature Learning in Infinite-Width Neural Networks},
  author={Greg Yang and J. Edward Hu},
  booktitle={International Conference on Machine Learning},
  year={2021},
}

@inproceedings{bordelonpehlevan2022,
  author    = {Bordelon, Blake and Pehlevan, Cengiz},
  title     = {Self-consistent dynamical field theory of kernel evolution in wide neural networks},
  author={Blake Bordelon and Cengiz Pehlevan},
  journal={Journal of Statistical Mechanics: Theory and Experiment},
  year={2022},
  volume={2023},
}

@article{lauditi2025,
  title={Adaptive kernel predictors from feature-learning infinite limits of neural networks},
  author={Clarissa Lauditi and Blake Bordelon and Cengiz Pehlevan},
  journal={ArXiv},
  year={2025},
  volume={abs/2502.07998},
}

@inproceedings{rubin2024,
  author    = {Rubin, Noa and Ringel, Zohar and Seroussi, Inbar and Helias, Moritz},
  title     = {A unified approach to feature learning in Bayesian neural networks},
  booktitle = {High-dimensional Learning Dynamics (HiLD) Workshop},
  year      = {2024}
}

@article{seroussi2023,
  title={Separation of scales and a thermodynamic description of feature learning in some CNNs},
  author={Inbar Seroussi and Zohar Ringel},
  journal={Nature Communications},
  year={2021},
  volume={14},
}

@inproceedings{jacot2018,
  author    = {Jacot, Arthur and Gabriel, Franck and Hongler, Cl{\'e}ment},
  title     = {Neural Tangent Kernel: convergence and generalization in neural networks},
  booktitle = {Advances in Neural Information Processing Systems},
  year      = {2018}
}

@inproceedings{lee2018,
  author    = {Lee, Jaehoon and Bahri, Yasaman and Novak, Roman and Schoenholz, Samuel S. and Pennington, Jeffrey and Sohl-Dickstein, Jascha},
  title     = {Deep neural networks as {Gaussian} processes},
  booktitle = {International Conference on Learning Representations},
  year      = {2018}
}

@inproceedings{danielymalach2020,
  author    = {Daniely, Amit and Malach, Eran},
  title     = {Learning parities with neural networks},
  booktitle = {Advances in Neural Information Processing Systems},
  year      = {2020}
}

@inproceedings{barak2022,
  author    = {Barak, Boaz and Edelman, Benjamin L. and Goel, Surbhi and Kakade, Sham and Malach, Eran and Zhang, Cyril},
  title     = {Hidden progress in deep learning: {SGD} learns parities near the computational limit},
  booktitle = {Advances in Neural Information Processing Systems},
  year      = {2022}
}

@article{leshno1993,
  author  = {Leshno, Moshe and Lin, Vladimir Ya. and Pinkus, Allan and Schocken, Shimon},
  title   = {Multilayer feedforward networks with a nonpolynomial activation function can approximate any function},
  journal = {Neural Networks},
  volume  = {6},
  number  = {6},
  pages   = {861--867},
  year    = {1993}
}

@inproceedings{buzaglo2024,
  author    = {Buzaglo, Gon and Harel, Itamar and Nacson, Mor Shpigel and Brutzkus, Alon and Srebro, Nathan and Soudry, Daniel},
  title     = {How Uniform Random Weights Induce Non-uniform Bias: Typical Interpolating Neural Networks Generalize with Narrow Teachers},
  booktitle = {Proceedings of the 41st International Conference on Machine Learning},
  series    = {Proceedings of Machine Learning Research},
  volume    = {235},
  pages     = {5035--5081},
  year      = {2024}
}

@inproceedings{danielygranot2019,
  author    = {Daniely, Amit and Granot, Elad},
  title     = {Generalization Bounds for Neural Networks via Approximate Description Length},
  booktitle = {Advances in Neural Information Processing Systems},
  year      = {2019}
}

@inproceedings{danielygranot2024,
  author    = {Daniely, Amit and Granot, Elad},
  title     = {On the Sample Complexity of Two-Layer Networks: Lipschitz vs. Element-Wise Lipschitz Activation},
  booktitle = {Proceedings of The 35th International Conference on Algorithmic Learning Theory},
  series    = {Proceedings of Machine Learning Research},
  volume    = {237},
  pages     = {505--517},
  year      = {2024}
}

@article{navehringel2021,
  author  = {Naveh, Gadi and Ringel, Zohar},
  title   = {A self consistent theory of Gaussian Processes captures feature learning effects in finite {CNN}s},
  journal = {arXiv preprint arXiv:2106.04110},
  year    = {2021}
}

@inproceedings{rubin2025kernels,
  author    = {Rubin, Noa and Fischer, Kirsten and Lindner, Javed and Seroussi, Inbar and Ringel, Zohar and Kr{"a}mer, Michael and Helias, Moritz},
  title     = {From Kernels to Features: A Multi-Scale Adaptive Theory of Feature Learning},
  booktitle = {Proceedings of the 42nd International Conference on Machine Learning},
  series    = {Proceedings of Machine Learning Research},
  volume    = {267},
  pages     = {52225--52257},
  year      = {2025}
}

@article{rubin2025curse,
  author  = {Rubin, Noa and Davidovich, Orit and Ringel, Zohar},
  title   = {Mitigating the Curse of Detail: Scaling Arguments for Feature Learning and Sample Complexity},
  journal = {arXiv preprint arXiv:2512.04165},
  year    = {2025}
}

@article{davidovichringel2026,
  author  = {Davidovich, Orit and Ringel, Zohar},
  title   = {Algorithmic Task Capture, Computational Complexity, and Inductive Bias of Infinite Transformers},
  journal = {arXiv preprint arXiv:2603.11161},
  year    = {2026}
}

@article{helias2026,
  author  = {Helias, Moritz and Lindner, Javed and Schutzeichel, Lars and Ringel, Zohar},
  title   = {Lecture notes: From Gaussian processes to feature learning},
  journal = {arXiv preprint arXiv:2602.12855},
  year    = {2026}
}

@article{zavatoneveth2025summary,
  author  = {Zavatone-Veth, Jacob A. and Bordelon, Blake and Pehlevan, Cengiz},
  title   = {Summary statistics of learning link changing neural representations to behavior},
  journal = {arXiv preprint arXiv:2504.16920},
  year    = {2025}
}

\end{document}